\pdfoutput=1

\documentclass{article} 
\usepackage{iclr2027_conference,times}

\usepackage{amsmath,amsfonts,bm}

\def\eqref#1{equation~\ref{#1}}

\def\1{\bm{1}}

\DeclareMathAlphabet{\mathsfit}{\encodingdefault}{\sfdefault}{m}{sl}
\SetMathAlphabet{\mathsfit}{bold}{\encodingdefault}{\sfdefault}{bx}{n}

\usepackage{hyperref}
\usepackage{url}

\usepackage{amsmath}
\usepackage{amssymb}
\usepackage{mathtools}
\usepackage{amsthm}

\theoremstyle{plain}
\newtheorem{theorem}{Theorem}[section]

\newtheorem{lemma}[theorem]{Lemma}
\newtheorem{corollary}[theorem]{Corollary}
\theoremstyle{definition}

\newtheorem{assumption}[theorem]{Assumption}
\theoremstyle{remark}

\title{Convergence Guarantees of Gradient Descent for Neural Networks via Generalized Lipschitz Smoothness}

\author{Siqiao Mu  \\
Department of Engineering Sciences \\
and Applied Mathematics\\
Northwestern University\\
Evanston, IL 60208, USA \\
\texttt{siqiaomu2026@u.northwestern.edu} \\
\And
Diego Klabjan \\
Department of Industrial Engineering \\ and Management Sciences\\
Northwestern University \\
Evanston, IL 60208, USA \\
\texttt{d-klabjan@northwestern.edu} \\
}

\iclrfinalcopy 

\begin{document}

\maketitle

\begin{abstract}
We establish the first convergence guarantees of gradient descent for general feedforward neural networks of any width or depth, with any initialization or dataset. We only assume that the activation functions are linearly bounded, Lipschitz continuous, and Lipschitz smooth---properties that hold for linear, tanh, softplus, sigmoid, and smoothed ReLU functions---and that the loss function is Lipschitz smooth in the \textit{model outputs}, a mild condition satisfied by mean squared error and binary cross-entropy loss. By relating the Lipschitz properties of one layer to the next, we obtain a novel generalized Lipschitz smoothness condition for an $L$-layer neural network where the change in gradient is upper bounded by the change in the parameter space, multiplied by a polynomial of the parameter norms at both endpoints of  degree $2L - 2$. This yields a descent lemma where the loss decreases as long as the learning rate is small enough with respect to the parameter norms. By ensuring that the leading term of the polynomial grows at a controlled rate, we prove that the minimum squared gradient norm converges to zero in $T$ iterations at rate $O(1/T^{1/L})$. 
\end{abstract}

\section{Introduction}

The behavior of gradient descent on the loss functions of neural networks has evaded a complete theoretical understanding for decades. These loss landscapes are not only highly nonconvex, but more critically, they do not satisfy  \textit{Lipschitz smoothness}, where the change in gradient is linearly bounded by the change in parameter space. Instead, the local Lipschitz constant, which dictates the sharpness of the landscape, can grow arbitrarily large as the parameters approach infinity. Because Lipschitz smoothness is a key condition for proving the convergence of gradient algorithms, a fundamental gap has emerged in machine learning research: popular optimization methods are analyzed under the smoothness assumption, yet empirically validated on neural network landscapes that are not theoretically well-characterized.

A rich body of research examines the convergence of gradient descent to \textit{global minima} of extremely wide or overparametrized neural networks \citep{Li, du2018gradient, pmlr-v97-allen-zhu19a, pmlr-v206-xu23c, taheriiclr, arora2018a, chatterjee2026convergencegradientdescentdeep,dana2025convergence}, with additional constraints on the architecture, activation functions, proportion of trained parameters, initialization scheme, or dataset. Under such conditions, the iterates remain within a bounded set close to the initialization, effectively satisfying \textit{local} Lipschitz smoothness. While these analyses provide valuable insights into how gradient descent can achieve unexpectedly strong performance in deep learning, they do not explain the behavior of gradient descent beyond these strict assumptions. In particular, the overparametrized or infinite-width regime  generally only captures the behavior of \textit{lazy training}, where the model effectively behaves like a fixed kernel machine or linear model. In reality, finite-width neural networks undergo \textit{feature learning}, where the weights may change significantly as the model learns feature representations of the data \citep{arora2019, yehudaiNEURIPS2019_5481b2f3, yang2022featurelearninginfinitewidthneural, vyas2023empirical, chizatlazy, pmlr-v125-woodworth20a}.

A separate line of research seeks to identify weaker requirements, such as \textit{non-uniform} or \textit{generalized Lipschitz} smoothness conditions, that still allow gradient descent to converge. For twice-differentiable functions,  Lipschitz smoothness is equivalent to requiring that the Hessian norm is bounded by a constant. One relaxation is ($L_0$, $L_1$)-smoothness \citep{Zhang2020Why}, which requires that the Hessian norm is upper bounded by an affine function of the gradient norm.  Other works consider bounding the Hessian with any nondecreasing function of the gradient norm \citep{li2023convex} or an affine function of the objective value \citep{vaswani2026convergencesteepestdescentadam}. A related concept is \textit{relative smoothness}, where the Hessian is upper bounded by the Hessian of some convex reference function, allowing for convergence of gradient descent or mirror descent \citep{Bauschke2017, lunesterov}. However, while these conditions can capture more complex settings, they are only \textit{empirically} validated, and are not actually satisfied by the loss landscapes of deep neural networks.

In this work, we provide a fresh analytical framework addressing the limitations of the above two approaches. We explicitly derive a generalized Lipschitz smoothness condition, which we call \textit{double polynomial smoothness}, that fully characterizes $L$-layer feedforward neural networks of any width, as long as the activation functions are linearly bounded, Lipschitz continuous, and Lipschitz smooth, and the loss is Lipschitz smooth in the model outputs\footnote{In particular, we use ``double" to highlight the dependence on the parameter norms at \textit{both} endpoints, and also to distinguish from the use of the term ``polynomial smoothness" in adjacent research fields such as the analysis of Approximate Message Passing algorithms \citep{nguyen2023stabilityapproximatemessagepassing}.}. By recursively relating the Lipschitz properties of one layer to the next, we establish that the change in model function gradient is upper bounded by the change in parameter space, multiplied by a polynomial of the parameter norms at both endpoints of degree $2L - 2$. This property can be further passed to the loss function, such that by integrating over the line segment between the endpoints, we obtain a novel descent lemma where the loss decreases as long as the learning rate is small enough with respect to the  parameter norms and loss at each step. The last step is showing with induction that the leading term of the polynomial grows at most sublinearly to infinity, at rate $T^{\frac{L-1}{L}}$. This allows us to prove that the minimum squared gradient norm converges to zero at rate $O(1/T^{1/L})$.

We examine the practical implications of our theory by conducting experiments in training feedforward models on the CIFAR-10 image classification dataset. We observe that scaling the learning rate according to our theory leads to faster and more stable convergence compared to constant learning rates by correctly adapting to the local geometry. Moreover, we observe that increasing the number of layers indeed slows the convergence, suggesting that the derived dependence on $L$ is fundamental and not just an artifact of the proof. Code is open-sourced at the Github repository at \url{https://github.com/siqiaomu/nn_optimization}.

Our convergence analysis applies to a broad range of feedforward neural networks, including models with linear, tanh, softplus,  sigmoid, and smoothed ReLU (GELU and SiLU) activation functions, and common loss functions including mean squared error, binary cross-entropy, and Huber loss. The advantage of our approach is that it is much easier to impose conditions such as Lipschitz smoothness or continuity on the layer-wise functions as opposed to the overall loss function. Unlike prior work, we have \textit{no requirements on the width, depth, initialization or dataset.} As we do not restrict the iterates to a bounded set, our analysis seamlessly accounts for the feature learning regime of neural networks, where the parameters may progress arbitrarily far from initialization, even towards a stationary point at infinity.

Our contributions can be summarized as follows. 
\begin{itemize}
    \item We explicitly derive the \textit{double polynomial smoothness} of feedforward neural networks and their loss functions, proving that for an $L$-layer neural network, the change in  gradient is upper bounded by the change in parameter space, multiplied by a polynomial of the parameter norms of both endpoints of degree $2L - 2$. 
    \item We establish the first convergence rate of gradient descent on neural networks with no conditions on the initialization, dataset, depth or width. We prove that gradient descent converges at rate $O(1/T^{1/L})$, provided that the learning rate is small enough with respect to the parameter norm and loss value at each step.
\end{itemize}

\section{Related Work}

\paragraph{Convergence to global minima of neural network loss functions.} Due to the enormous success of neural networks, a large body of work focuses on identifying  circumstances under which gradient descent converges to \textit{global minima} of neural network loss functions, where the resulting model achieves zero error on the training set. These analyses typically rely on a patchwork of assumptions on the model dimension, initialization or dataset in order to demonstrate that some gradient dominance condition persists for all points in the gradient descent trajectory (although not globally), driving the loss to zero. Many such results hold in an over-parametrized regime where the hidden layers are extremely wide relative to the number of training samples \citep{Li, du2018gradient, pmlr-v97-allen-zhu19a, dana2025convergence, pmlr-v206-xu23c, taheriiclr}. Still others consider \textit{infinite-width} neural networks \citep{https://doi.org/10.1002/cpa.22200, barboni2026traininginfinitelydeepwide}; in particular, one central framework is the neural tangent kernel (NTK) analysis, where as the width of the network goes to infinity, the gradient descent training process reduces to a kernel method. In these analyses, the iterates can be shown to remain in a bounded set close to initialization. They also often rely on random initialization; for example, the NTK analysis leverages the fact that an infinite-width neural network with random initializations converges to a Gaussian process by way of the Central Limit Theorem \citep{3327757.3327948}.
Finally, some works may require additional restrictions on the dataset; for example the works \cite{arora2018a} and  \cite{11568902} consider gradient descent for deep linear neural networks on whitened data, and \cite{chatterjee2026convergencegradientdescentdeep} analyzes a complementary setting where the neural network dimension is arbitrary, but the dimension of the input data is at least the number of data points. While these analyses provide important insights into how neural networks can achieve near-zero training loss under strict assumptions, they leave the broader behavior of gradient descent beyond such assumptions unexplained.


At the same time, a general analysis of the convergence rate of gradient descent to a \textit{stationary point} for neural networks--independent of architecture, dataset or initialization---is conspicuously missing. A well-known issue regarding 
this setting is the lack of a global Lipschitz constant or bound on the Hessian norm as the iterates converge towards infinity \citep{sun2019optimizationdeeplearningtheory}. By contrast, if training remains within a compact set, the Hessian is necessarily bounded for all iterates, and the function is locally Lipschitz smooth. Accounting for the potentially unbounded behavior of the parameters is not just a technical challenge; it is critical to understanding the \textit{feature learning} regime of neural networks, where the parameters change significantly from initialization as the model learns feature representations of the data. In contrast, it has been shown that the lazy training regime, where iterates remain close to initialization, leads to underperforming models and does not fully explain the power of deep learning \citep{chizatlazy, pmlr-v125-woodworth20a, pmlr-v145-seleznova22a}. 

\paragraph{Comparison with \cite{wang2026convergencegradientdescentgeneral}.} The work \cite{wang2026convergencegradientdescentgeneral} first observed that neural networks satisfy a polynomial smoothness condition that can be conveniently derived from recursively relating the Lipschitz properties of the activation functions at each layer. They prove that the loss function satisfies \textit{polynomial generalized smoothness}, a similar but less specific version of double polynomial smoothness. However, they only utilize this condition to uniformly bound the  Lipschitz constant on a ball of radius $R$ around a stationary point, with the constant depending on $R$, the parameter norm of the stationary point, and the maximum gradient norm over the ball. Setting the learning rate small enough with respect to this constant, they prove \textit{asymptotic} convergence to a stationary point as long as the neural network is sufficiently overparametrized, the data distribution is nondegenerate, and the (mean squared error) loss function satisfies a local dissipative condition. Critically, they require that the width of the neural network exceeds the number of training data samples multiplied by the data dimension, which is looser than some NTK assumptions but still highly unrealistic.

In contrast, our analytical framework has \textit{no} conditions on the network width or dataset, and it applies to a broad range of loss functions. Instead of constraining the parameters to a bounded set, we carefully track the potentially unbounded growth of the parameter norms---corresponding to the \textit{dynamic} local Lipschitz constant---to achieve convergence to a stationary point even if at infinity. Key aspects of our analysis rely on the explicit smoothness expression we derive, including the specific \textit{degree} of the polynomial in terms of $L$, which enables us to establish the non-asymptotic $O(1/T^{1/L})$ convergence rate. In  Appendix \ref{sec:wang}, we further compare and contrast the works and discuss in detail why the condition in \cite{wang2026convergencegradientdescentgeneral} inhibits the same kind of convergence analysis.

\paragraph{Generalized Lipschitz smoothness.} In the interest of developing more faithful analyses of the neural network setting, a recent line of work has targeted relaxations of the Lipschitz smoothness condition that can still yield convergence of gradient algorithms, also known as \textit{generalized Lipschitz smoothness} or \textit{non-uniform smoothness}. The work \cite{Zhang2020Why} proposes the $(L_0, L_1)$-smoothness condition, where the Hessian norm is bounded by an affine function of the gradient norm, to explain the performance of gradient clipping. This condition was then extended to polynomial functions \citep{NEURIPS2019_7a2b33c6, yang2025adaptive} and then general nondecreasing functions of the gradient norm \citep{li2023convex}. Additional works have considered bounding the Hessian norm by polynomial functions of the objective value \citep{vaswani2026convergencesteepestdescentadam, pmlr-v139-mei21a}. However, these generalized  conditions are typically \textit{empirically} validated and are only mathematically satisfied by some special cases of neural networks---for example, two-layer neural networks where one layer is frozen \citep{vaswani2026convergencesteepestdescentadam}, two-layer neural networks with a strictly increasing activation function on linearly separable data \citep{taheri2023fastconvergencelearningtwolayer}, or deep linear neural networks where the weights remain in a ``strongly balanced" subspace at all times \citep{alimisis2026needwarmuptheoreticalperspective}.

A related area of research is \textit{relative smoothness}, motivated by mirror descent analysis, where the Hessian norm is upper bounded by the Hessian of some convex reference function that is used to map iterates to and from a dual space \citep{Bauschke2017, lunesterov}. In particular, the works \citep{lunesterov, pmlr-v238-fatkhullin24a} identify that a two-layer linear neural network satisfies a relative smoothness condition where the loss Hessian is bounded by a quadratic term in the parameter norms. This condition is close to the double polynomial smoothness presented in our work. However, they only address mirror descent using this specific reference function, not vanilla gradient descent, and they do not consider general neural networks with nonlinear activation functions.

Finally, some works attempt to circumvent issues with global smoothness by considering separately the local conditions around the gradient descent trajectory. The work \cite{mishkinNEURIPS2024_1ac83203} proposes \textit{directional smoothness}, where the gradient variation depends solely on the local conditions on the optimization path. However, this requires solving implicit equations to obtain the learning rate, and the existing analysis only applies to convex functions. The work \cite{fox2026glocal} introduces the notion of ``glocal" smoothness, where the loss function may have a prohibitively large \textit{global} Lipschitz constant but a smaller \textit{local} Lipschitz constant. This enables the analysis of algorithms on more complex landscapes, but still excludes neural networks, which lack a global Lipschitz constant over an unbounded domain. Finally, the work \cite{berahas2024nonuniform} considers the broad scenario where one has access to a local first-order smoothness oracle that bounds the local Lipschitz smoothness constant within a ball of some radius $R$. However, a key challenge of this framing lies in choosing $R$ at each step such that it includes the next iterate. In contrast, the double polynomial smoothness derived in this work holds \textit{globally} for any two points regardless of their distance apart.

\paragraph{Matrix factorization and low-rank fine-tuning.} Prior works have observed the close connection between  matrix factorization  and neural network training \citep{sun2019optimizationdeeplearningtheory, xiong2024how}. In particular, we build on the convergence proof structure of the Low-Rank Adaptation (LoRA) algorithm \citep{mu2026convergencerateloragradient}, which can be reinterpreted as training a specific two-layer linear neural network. However, the general neural network setting is significantly more complex due to the nonlinearity of the activation functions and the recursive relationships generated by the number of layers.

\section{Setup}

Let $\odot$ denote the Hadamard product and $\lVert\cdot\rVert_F$ denote the Frobenius norm, where for the real matrix $X \in \mathbb{R}^{m \times n}$, we have $\lVert X\rVert_F = \sqrt{\sum_{i=1}^m \sum_{j=1}^n X_{ij}^2}$. We also write $\lVert \cdot \rVert$ for $\lVert \cdot \rVert_F$ throughout. We denote $\langle A,B\rangle_F =\operatorname{tr}(A^TB)$.

We consider empirical risk minimization on a dataset of $n$ training samples $\{x_i, y_i \}^n_{i = 1}$, where $x_i \in \mathbb{R}^d$ and $y_i \in \mathbb{R}$. Let the constant $c_x = \max \{ 1, \max_{i \in \{ 1,..., n\}} \lVert x_i \rVert_\infty \}$ represent an upper bound on any individual component of any input data, such that for all $i$, we have $\lVert x_i \rVert \leq c_x d^{1/2}$.

We model the data with $f$, a feedforward neural network. Denote by $\sigma_{\ell}: \mathbb{R} \to \mathbb{R}$ the activation function at the $\ell$th layer, and we overload notation such that for $X\in \mathbb{R}^{n \times m}$,  $\sigma_\ell(X) \in \mathbb{R}^{n \times m}$ denotes the element-wise application of the activation function. Then an $L$-layer feedforward neural network $f$, predicting on data sample $x$, is parametrized by $L$ weight matrices $w_1, w_2,..., w_L$ as 
\begin{equation*}
    f(x; w_1, w_2,...,w_L) = \sigma_L(w_{L} \sigma_{L-1}(w_{L-1} \sigma_{L - 2}(... w_2 \sigma_1 (w_1 x)))),
\end{equation*}
where $L \geq 2$, 
$w_\ell \in \mathbb{R}^{d_\ell \times d_{\ell - 1}}$, $d_0 = d$, and $d_{L} = 1$. We also denote the width of the neural network as $d_{max} = \max \{d_0,..., d_L\}$. For the following analysis, we assume that the activation functions are linearly bounded, Lipschitz continuous, and Lipschitz smooth.

\begin{assumption}
\label{assump:lipschitz}
    For all $\ell$, the activation function $\sigma_\ell: \mathbb{R} \to \mathbb{R}$ is continuously  differentiable and linearly bounded such that for all $w \in \mathbb{R}$, we have 
    \begin{equation}
    \label{eq:l0}
        |\sigma_\ell(w)| \leq c_0 + c_1 |w|.
    \end{equation}
    Moreover, $\sigma_\ell$ is Lipschitz continuous and Lipschitz smooth with nonnegative constants $c_2$, $c_3$ such that for all  $w_1, w_2 \in \mathbb{R}$, we have
    \begin{equation}
        \label{eq:l2}
        |\sigma_\ell(w_1) - \sigma_\ell(w_2)| \leq c_2 |w_1 - w_2|, \qquad |\sigma'_\ell(w_1) - \sigma'_\ell(w_2)| \leq c_3 |w_1 - w_2|.
    \end{equation}

\end{assumption}

\noindent Assumption \ref{assump:lipschitz} also implies that $|\sigma'_\ell(w)| \leq c_2$ for all $w \in \mathbb{R}$. To simplify the algebra, we state our result in terms of $c_{max} =  \max\{c_0, c_1, c_2, c_3, 1\},$ where $c_{max} = 1$ for standard tanh, sigmoid, linear, and softplus functions, and $c_{max}$ is slightly greater than $1$ for GELU and SiLU functions (see Appendix \ref{append:Lipproperties}).

We construct a block matrix $W$ with the weight matrices $w_1,..., w_L$ on the diagonal as follows,
\[
W = \begin{bmatrix}
w_1 & &  \\
& \ddots & \\
 & & w_L
\end{bmatrix}, \quad W \in \mathbb{R}^{D_2 \times D_1}, \quad \text{ where } D_1 = \sum_{\ell = 0}^{L-1} d_\ell, D_2 = \sum_{\ell = 1}^{L} d_\ell.
\]

For each layer $\ell$, define the matrices $e_{\ell1} \in \mathbb{R}^{d_{\ell} \times D_2}$ and $e_{\ell2} \in \mathbb{R}^{D_1 \times d_{\ell -1 }}$ as follows
\begin{align*}
    e_{\ell1} = &
    \begin{bmatrix}
    0_{d_\ell \times d_1} & 0_{d_\ell \times d_2} & \hdots & 0_{d_{\ell} \times  d_{\ell - 1}} & I_{d_\ell \times d_\ell} & 0_{d_\ell \times d_{\ell + 1}} & \hdots & 0_{d_\ell \times d_{L}} 
    \end{bmatrix}, \\
    e_{\ell2} = &
    \begin{bmatrix}
    0_{d_{\ell-1} \times d_0} & 0_{d_{\ell-1} \times d_1} & \hdots & 0_{d_{\ell-1} \times  d_{\ell - 2}} & I_{d_{\ell-1} \times d_{\ell-1}} & 0_{d_{\ell-1} \times d_{\ell}} & \hdots & 0_{d_{\ell -1} \times d_{L - 1}} 
    \end{bmatrix}^T.
\end{align*}
We define the linear operator $E_\ell[W]$ that extracts the $\ell$th diagonal block from $W$, recovering  $w_\ell$, as well as its corresponding adjoint operator $E_{\ell}^*[M]$ for $M \in \mathbb{R}^{d_\ell \times d_{\ell - 1}}$ 
\begin{equation}
\label{eq:defineE}
    E_{\ell}[W] = e_{\ell 1} W e_{\ell 2}= w_\ell, \qquad E^*_\ell[M] = e^T_{\ell 1} M e^T_{\ell 2}.
\end{equation}

Now we can write $f$ in a reparametrized form $F$ as a function of $W \in \mathbb{R}^{D_2 \times D_1}$. We have
\begin{equation}
\label{eq:restructured_f}
    f(x_i; w_1,..,w_L) = F(x_i; W) = \sigma_L (E_L [W] \sigma_{L-1} (... E_2 [W]  \sigma_1
    (E_1[W] x_i))).
\end{equation}

Let $\hat{y}_i = F(x_i; W)$ represent the prediction on the sample $(x_i, y_i)$. We define the individual loss on the output data $y_i$ as $J_i: \mathbb{R} \to \mathbb{R}$. Then the overall loss function $\mathcal{L}:\mathbb{R}^{D_2 \times D_1} \to \mathbb{R}$ is given as
\begin{align}
\label{eq:loss}
    \mathcal{L}(W) = &  \frac{1}{n} \sum_{i = 1}^n J_i (\hat{y}_i) = \frac{1}{n} \sum_{i = 1}^n J_i (F(x_i; W)).
\end{align}

We require a mild smoothness assumption on $J_i$,  which is satisfied by common loss functions such as mean squared error (MSE), binary cross-entropy (BCE), and Huber loss (see Appendix \ref{append:Lipproperties}).  
\begin{assumption}
\label{assump:lipJ}
    Let $c_J \geq 1$. The individual loss $J_i$ is nonnegative and $c_J$-Lipschitz smooth in the model outputs, such that for $u$, $v \in \mathbb{R}$ representing different model outputs, we have
        $|J_i'(u)~-~J_i'(v)| \leq c_J | u - v | $.
\end{assumption}
For MSE, $J_i(u) = (u - y_i)^2$, and $J_i'(u) = 2(u - y_i)$, so it satisfies this assumption with $c_J = 2$. The BCE loss function also satisfies Lipschitz smoothness with respect to the \textit{logit} outputs before the softmax function. Specifically, for a model where the last activation function is linear, we define the BCE loss on the $i$th point as (see also, Appendix \ref{append:Lipproperties})
\begin{align*}
    J_i(u) = & - \big[ y_i \log (\frac{1}{1 + e^{-u}}) + (1 - y_i) \log(1 - \frac{1}{1 + e^{-u}}) \big] = - u y_i + \log(1 + e^u).
\end{align*}
When we restrict $y_i \in [0, 1]$, $J_i$ is nonnegative. The absolute value of the second derivative $J''_i(u)$ is upper bounded by $1/4$, so BCE satisfies the assumption and we can choose $c_J = 1$. In addition, multi-class cross entropy also satisfies Lipschitz smoothness with respect to multiple logits and may also be covered by our theory; however, the model gradient would be a tensor rather than a matrix, which would significantly complicate the notation. We therefore leave this case for future work.

We minimize $\mathcal{L}(W)$ using the gradient descent algorithm,
\begin{align}
\label{eq:gd}
    W_{t+1} = W_t - \eta_t \nabla \mathcal{L}(W_t),
\end{align}
where $\nabla \mathcal{L}(W) =   \frac{1}{n} \sum_{i = 1}^n J_i'(F(x_i;W))\cdot \nabla_W F(x_i ; W).$
The evolution of the block diagonal elements of $W$ under gradient descent matches that of the original weight matrices $w_1,...,w_L$, while the off diagonal elements exhibit zero change because they do not affect the value of $F$ or $\mathcal{L}$.

\section{Analyses}

Broadly speaking, we aim to show that for any $W_1$, $W_2$, we can bound $\lVert \nabla \mathcal{L}(W_2) - \nabla \mathcal{L}(W_1) \rVert$ relative to $\lVert W_2 - W_1 \rVert$. We first relate the change in loss  gradient to the change in model function value and change in model gradient, as stated in Lemma \ref{lemma:LLsmooth}.

\begin{lemma}
\label{lemma:LLsmooth}
    For $W_1, W_2 \in \mathbb{R}^{D_2 \times D_1}$, we have
    \begin{align*}
        \lVert \nabla \mathcal{L}(W_2) - \nabla \mathcal{L}(W_1) \rVert \leq & \frac{1}{n} \sum_{i = 1}^n  c_J|F(x_i ; W_2) - F(x_i ; W_1) |  \lVert\nabla_W F(x_i ; W_2) \rVert \\
    &+ \frac{1}{n} \sum_{i = 1}^n | J_i'(F(x_i;W_1))| \lVert \nabla_W F(x_i ; W_2) - \nabla_W F(x_i ; W_1) \rVert .
    \end{align*}
\end{lemma}
Now we explicitly compute $\nabla_W F(x_i ; W)$ for some  $x_i$. Defining $h_0(W) = x_i$, we recursively define the intermediate variables $z_\ell(W) \in \mathbb{R}^{d_\ell}$ and $h_\ell(W) \in \mathbb{R}^{d_{\ell}}$ for $\ell = 1,..., L$ as follows
\begin{equation}
\label{eq:definez}
    z_\ell(W) =  E_\ell[W] h_{\ell - 1}(W), \qquad
     h_\ell(W) = \sigma_\ell(z_\ell(W)).
\end{equation}
Then we have $F(W) = h_L(W) = \sigma_L(z_L(W))$. We also define the variable $\delta_{\ell} \in \mathbb{R}^{d_\ell}$ as  
\begin{equation}
\label{eq:deltadefine}
    \delta_L(W) = \sigma'_L(z_L(W)), \qquad \delta_{\ell }(W) = \sigma'_\ell(z_{\ell }(W)) \odot (E_{\ell + 1}  [W]^T \delta_{\ell + 1}(W)).
\end{equation}
\begin{lemma}
\label{lemma:gradformula}
    For some input 
    data $x_i$, the gradient of $F$ with respect to $W$ is given as
    \begin{equation*}
        \nabla_W F(x_i ; W) = \sum_{\ell = 1}^{L}  E^*_{\ell}[\delta_{\ell} (W) h_{\ell - 1}(W)^T].
    \end{equation*}
\end{lemma}

We decompose the change in $\nabla_W F(x_i ; W)$ in terms of changes in $\delta_{\ell}$ and $h_{\ell}$, which can be defined recursively in terms of other layers. Because of the Lipschitz properties of each $\sigma_\ell$ from Assumption \ref{assump:lipschitz}, these terms can also be bounded by $\lVert W_2 - W_1 \rVert$. Collecting these terms together allows us to establish the \textit{double polynomial smoothness} of $F$, where the change in gradient is upper bounded by $\lVert W_2 - W_1 \rVert$, multiplied by a polynomial of the endpoints $\lVert W_1 \rVert$ and $\lVert W_2 \rVert$.

\begin{lemma}
\label{lemma:lip_f}
    For all $W_1$, $W_2 \in \mathbb{R}^{D_2 \times D_1}$, and $x_i$, we have
    \begin{align*}
        \lVert F(x_i ; W_2 ) - F(x_i ; W_1) \rVert \leq c_x d_{max}^{1/2} c_{max}^L \lVert W_2 - W_1 \rVert \sum_{j = 0}^{L - 1} \sum_{k = 0}^{L - j - 1}  \lVert W_1 \rVert^j \lVert W_2 \rVert^{k}.
    \end{align*}
\end{lemma}

\begin{lemma}
\label{lemma:lip_nablaf}
For all $W_1$, $W_2 \in \mathbb{R}^{D_2 \times D_1}$, and $x_i$, we have
\begin{align*}
    \lVert \nabla_W F(x_i ; W_2) - \nabla_W F(x_i ; W_1) \rVert \leq 4 c_x^2 d_{max} c_{max}^{3L} L^3 \lVert W_2 - W_1 \rVert \sum_{j = 0}^{2L - 2}  \lVert W_1 \rVert^j  \sum_{k = 0}^{2L - j- 2} \lVert W_2 \rVert^{k}.
\end{align*}

\end{lemma}

By plugging Lemmas \ref{lemma:lip_f} and \ref{lemma:lip_nablaf} into Lemma \ref{lemma:LLsmooth}, we have the following bound establishing the double polynomial smoothness of $\mathcal{L}$.
\begin{lemma}
\label{lemma:Lipschitzsmooth}
    For all $W_1, W_2 \in \mathbb{R}^{D_2 \times D_1}$, we have
    \begin{equation*}
        \lVert \nabla \mathcal{L}(W_2) - \nabla \mathcal{L}(W_1) \rVert \leq 8 c_J c_x^2  d_{max} c_{max}^{3L} L^3  \lVert W_2 - W_1 \rVert (1 + \mathcal{L}(W_1)^{1/2})   \sum_{j = 0}^{2L - 2}  \lVert W_1 \rVert^j  \sum_{k = 0}^{2L - j- 2} \lVert W_2 \rVert^{k}.
    \end{equation*}
\end{lemma}
Lemma \ref{lemma:Lipschitzsmooth} also includes a factor of $1 + \mathcal{L}(W_1)^{1/2}$ in addition to the polynomials of  $\lVert W_1 \rVert$, $\lVert W_2 \rVert$. We note that $\mathcal{L}(W_1)$ can also be upper bounded in terms of a polynomial of $\lVert W_1 \rVert$, but leaving it in this form allows us to leverage descent in $\mathcal{L}$ over time. In fact, due to this factor of $\mathcal{L}(W_1)$, the right hand side of Lemma \ref{lemma:Lipschitzsmooth} is \textit{asymmetric} in its dependence on $W_1$ and $W_2$. Even if we assume twice-differentiability, this kind of asymmetric smoothness condition is not equivalent to any pointwise bound on the Hessian. This departs from many prior  works, such as analyses of $(L_0, L_1)$-smoothness, which assume a direct equivalence between smoothness and Hessian bounds.

Lemma \ref{lemma:Lipschitzsmooth} directly leads to the following descent lemma. In particular, after integrating over the segment between $W_1$ and $W_2$, the polynomial terms in $\lVert W_2 \rVert$ translate to higher order terms of $\lVert W_2 - W_1 \rVert$. This differs from the classic Lipschitz smoothness descent lemma, which  only contains a quadratic term of $\lVert W_2 - W_1 \rVert$.

\begin{lemma}
\label{lemma:descent} For all $W_1$ and $W_2 \in \mathbb{R}^{D_2 \times D_1}$, we have
\begin{align*}
\mathcal{L}(W_2) - \mathcal{L}(W_1) \leq &\langle \nabla \mathcal{L}(W_1), W_2 - W_1 \rangle_F \\
    &+ C \Big[ (1 + \mathcal{L}(W_1)^{1/2})  
    \sum_{j = 0}^{2L - 2 }  \lVert W_1 \rVert^{j } \Big]  \sum_{k = 0}^{2L - 2} 
    \frac{1}{k + 2} \lVert  W_2 - W_1 \rVert^{k+2},
\end{align*}
where $C =  2^{2L + 2} c_J c_x^2 d_{max} c_{max}^{3L} L^4$.  
\end{lemma}
Applying this lemma to the gradient descent iterates $W_t$ and $W_{t+1}$ (\ref{eq:gd}) yields descent in the loss function value as long as $\eta_t$ is chosen appropriately. The formula for $\eta_t$ only requires access to the current parameter norm $\lVert W_t \rVert$ and loss value $\mathcal{L}(W_t)$ and is therefore easily computable without  any subroutines or implicit solutions.

\begin{lemma}
\label{lemma:onestep}
    For one step of gradient descent $W_{t+1} = W_t - \eta_t \nabla \mathcal{L}(W_t)$, if $\eta_t$ satisfies
\begin{equation}
\label{eq:lr}
    \eta_t = \frac{1}{\rho(1 +  \mathcal{L}(W_t)^{1/2})  
    \sum_{i = 0}^{2L - 2 }  \lVert W_t \rVert^{i } },
\end{equation}
where $\rho = 2^{2L + 3} c_J c_x^2 d_{max} c_{max}^{3L} L^5 $,
then the function value descends in one step,
\begin{equation*}
    \mathcal{L}(W_{t+1}) - \mathcal{L}(W_t) \leq - \frac{\eta_t}{2L} \lVert \nabla \mathcal{L}(W_t) \rVert^2,
\end{equation*}
and for $T$ steps we obtain the bound
\begin{equation}
\label{eq:cumulativenorms}
    \sum_{t = 0}^{T - 1} \eta_t \lVert \nabla \mathcal{L}(W_t) \rVert^2 \leq 2L \mathcal{L}(W_0).
\end{equation}

\end{lemma}

To show convergence of $\lVert \nabla \mathcal{L}(W_t) \rVert^2$ to zero, we require that the learning rate (\ref{eq:lr}) not decay too quickly. By Lemma \ref{lemma:onestep}, we can bound $\mathcal{L}(W_t) \leq \mathcal{L}(W_0)$. It remains to control the growth of $\sum_{i = 0}^{2L - 2 }  \lVert W_t \rVert^{i }$, which is dominated by the leading term $\lVert W_t \rVert^{2L - 2}$. We have for $N = 2L - 2$,
\begin{align*}
    \lVert W_T \rVert^N \leq & \lVert W_0 \rVert^N + N \sum_{t = 0}^{T - 1} \eta_t^{1/N} \lVert \nabla \mathcal{L}(W_t) \rVert + \sum_{t = 0}^{T - 1} \sum_{k = 2}^N \binom{N}{k}  \eta_t^{k/2} \lVert \nabla \mathcal{L}(W_t) \rVert^k. 
\end{align*}
By (\ref{eq:cumulativenorms}), it can be shown that the last term is bounded in terms of the constant $2L \mathcal{L}(W_0)$. However, the growth of $\sum_{t = 0}^{T - 1} \eta_t^{1/N} \lVert \nabla \mathcal{L}(W_t) \rVert$ depends on the value of $N$. For two-layer neural networks, $L = 2$ and $N = 2$, and we can directly use  (\ref{eq:cumulativenorms}) to show that this term grows as $O(
\sqrt{T})$, yielding an $O(1/\sqrt{T})$ convergence rate. However, for deeper neural networks $L > 2$, we have to bound the sum via Cauchy-Schwartz as follows,
\begin{align*}
    \sum_{t = 0}^{T - 1} \eta_t^{1/N} \lVert \nabla \mathcal{L}(W_t) \rVert &= \sum_{t = 0}^{T - 1} \eta_t^{1/N - 1/2} \eta_t^{1/2} \lVert \nabla \mathcal{L}(W_t) \rVert \leq \Big( \sum_{t = 0}^{T - 1} \eta_t^\frac{2 - N}{N} \Big)^{1/2} \Big( \sum_{t = 0}^{T - 1} \eta_t \lVert \nabla \mathcal{L}(W_t) \rVert^2 \Big)^{1/2}.
\end{align*}
Since $( \sum_{t = 0}^{T - 1} \eta_t \lVert \nabla \mathcal{L}(W_t) \rVert^2 )^{1/2}$ can also be bounded via (\ref{eq:cumulativenorms}), it remains to determine the growth of $( \sum_{t = 0}^{T - 1} \eta_t^\frac{2 - N}{N} )^{1/2}$.
We show by induction that  $( \sum_{t = 0}^{T - 1} \eta_t^\frac{2 - N}{N} )^{1/2}$ and therefore  $\lVert W_T \rVert^{2L - 2}$ grows as $O(T^{\frac{L-1}{L}})$, leading to an overall $O(1/T^{1/L})$ convergence rate for the minimum squared gradient norm. We state our main result as follows.

\begin{theorem}
\label{thm:generalNN}
Consider the loss function $\mathcal{L}$ (\ref{eq:loss}) of an $L$-layer neural network (\ref{eq:restructured_f}). Suppose Assumptions \ref{assump:lipschitz} and \ref{assump:lipJ} hold. Then after $T$ steps of gradient descent (\ref{eq:gd}), with $\eta_t$ set as (\ref{eq:lr}), we have
\begin{equation*}
    \min_{t = 0,...,T-1} \lVert \nabla \mathcal{L}(W_t) \rVert^2 = O \Big( \frac{1}{T^{1/L}} \Big).
\end{equation*}
\end{theorem}

\begin{corollary}
    If we have $\lVert W_t \rVert \leq R$ for all $t$, then after $T$ steps of gradient descent with the constant learning rate $\eta = \frac{1}{\rho(1 + \mathcal{L}(W_0)^{1/2}) \sum_{i = 0}^{2L - 2} R^i}$
    we have
    $\min_{t = 0,...,T-1} \lVert \nabla \mathcal{L}(W_t) \rVert^2 = O ( 1/T )$.
\end{corollary}
In the above statements, the $O(\cdot)$ notation hides polynomial dependence on $c_J$, $c_x$, $d_{max}$, $c_{max}$, $\mathcal{L}(W_0)$, and $\lVert W_0 \rVert$, and exponential dependence on $L$.
Theorem \ref{thm:generalNN} establishes the first upper bound of the convergence rate of gradient descent on neural network loss functions of any depth or width, for any initialization or dataset. Moreover, our analysis does not rely on requiring that the parameters remain within a bounded set; we demonstrate that gradient descent converges even if $\lVert W_T \rVert$ grows to infinity, capturing the feature learning regime of neural networks. As the number of layers $L$ increases, the local geometry (controlled by $\lVert W_t \rVert^{2L - 2}$) grows more sharp, forcing the learning rate to decay faster and slowing convergence. Whether this phenomenon represents a fundamental limitation or is an artifact of the analysis remains an open question.

\section{Experiments}

We investigate practical implications of our theory by training feedforward neural networks of varying depths on the CIFAR-10 image classification dataset \citep{cifar}. We fix the width of the internal layers as $512$, and we use the GELU activation function due to its close similarity to ReLU, a popular activation function. We use minibatch gradient descent with batch size $512$ on the cross-entropy loss function. For additional experimental details and results, including the Github repository, see Appendix \ref{sec:experimentappend}.

First, we demonstrate that scaling the learning rate according to (\ref{eq:lr}) leads to improved convergence. We define the following adaptive learning rate $\eta^{adapt}$, which scales with the parameter norm and batch loss value at each step, replacing the constants in (\ref{eq:lr}) with a tunable hyperparameter $\alpha$,
\begin{align*}
    \eta_t^{adapt} = \frac{\alpha \lVert W_0 \rVert^{2L - 2}}{ (1 + \mathcal{L}(W_t)^{1/2}) \sum_{j = 0}^{2L - 2}\lVert W_t \rVert^j}.
\end{align*}
Figure \ref{fig:fig1} displays the loss and parameter norms over the training steps for constant and adaptive learning rates, for a neural network of depth $L=4$ (for other values of $L$, see Figure \ref{fig:bigcomparison} in Appendix \ref{sec:experimentappend}). For all cases, $\eta^{adapt}_t$ displays faster and more stable convergence than constant learning rates of comparable values. In particular, $\eta^{adapt}_t$ adapts to the local geometry of the loss function by starting at a large value and decaying over time, which accelerates the iterates away from initialization.

\begin{figure}
    \centering
    \includegraphics[width=0.95\linewidth]{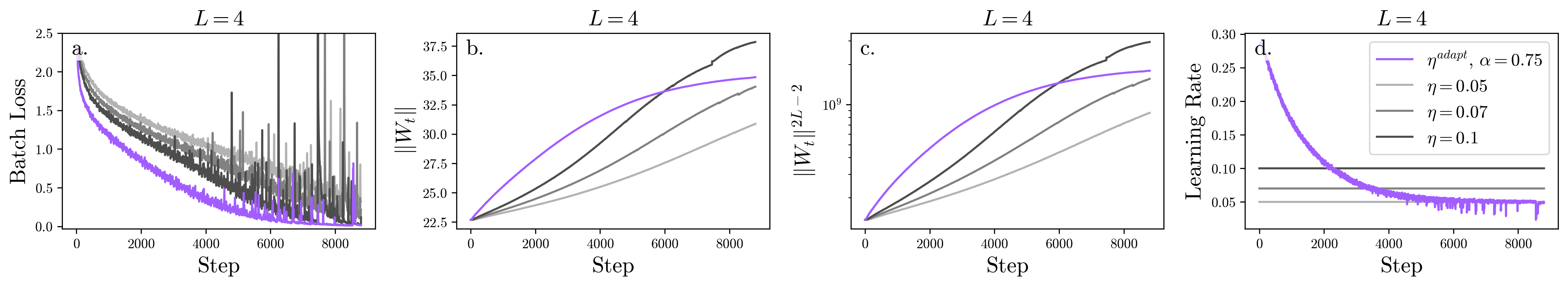}
    \caption{Adaptive vs. constant learning rate for standard neural network with $L = 4$.}
    \label{fig:fig1}
\end{figure}

Second, we examine whether increasing the number of layers in a neural network indeed slows down convergence, as suggested by Theorem \ref{thm:generalNN}. For standard neural networks, the effect of increasing depth is challenging to isolate because increasing the depth 
also increases the number of parameters, which may substantially change other factors such as the overfitting of the model or the size of the parameter norms. However, our convergence analysis actually applies to a more general class of problems than the standard neural network structure, as the specific choice of the projection operators does not matter to the proof as long as they satisfy certain contraction properties (Lemma \ref{lemma:projection}). This allows us to investigate the impact of depth for a \textit{shared-weights} neural network, where each layer shares roughly the same weight matrix. We construct such a neural network where for $d$-dimensional input data, the input layer is $512 \times d$, and all other layers are a fixed $512 \times 512$ subset of the input layer weights, except for the output layer which is a $1 \times 512$ subset. This model has a fixed size and number of parameters even as the number of layers increases.

\begin{figure}[t]
    \centering
\includegraphics[width=0.95\linewidth]{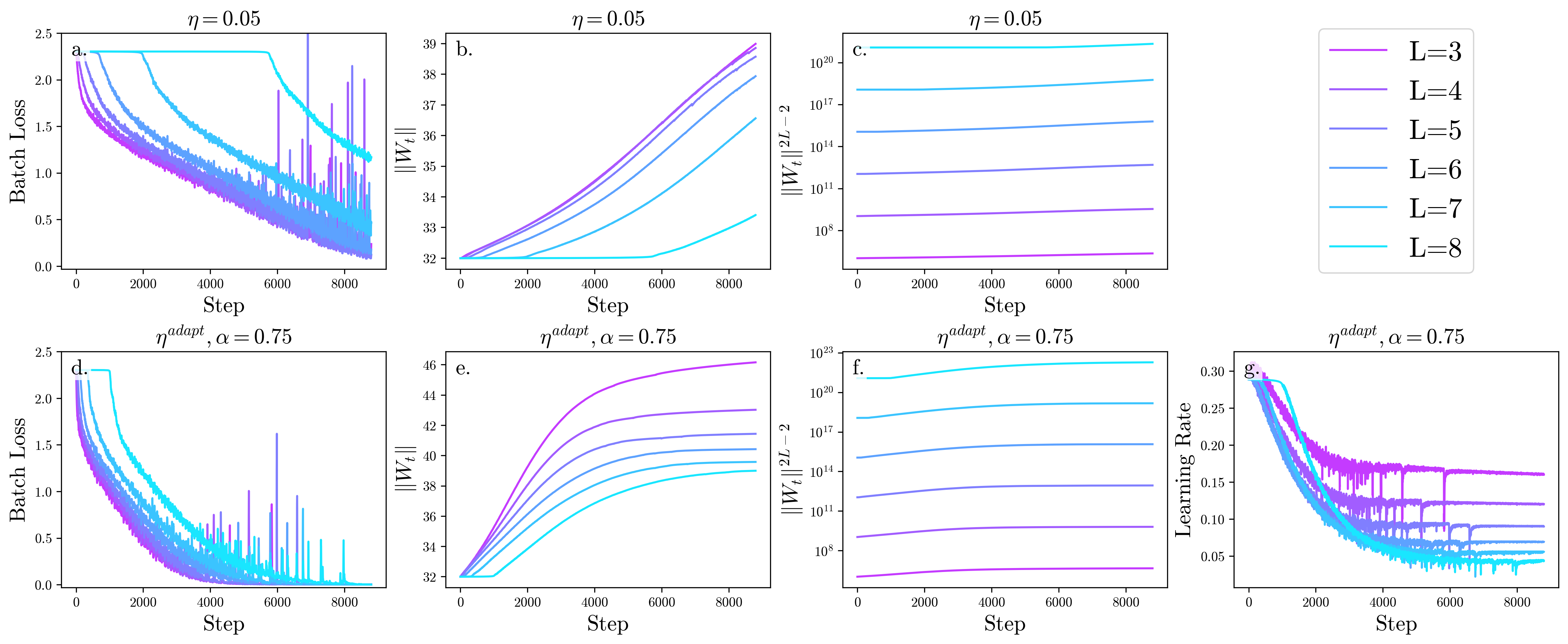}
    \caption{Training of shared-weights neural networks of varying $L$ with constant (top row) and adaptive (bottom row) learning rates for 100 epochs.}
    \label{fig:fig2}
    \vspace{-0.15in}
\end{figure}

Figure \ref{fig:fig2} displays the results of training the shared-weights models for increasing $L$. As in Figure \ref{fig:fig1}, the adaptive learning rate displays faster and more stable convergence than constant learning rates. Moreover, for both types of learning rates, we observe the convergence is indeed consistently slower for larger $L$. For example, the $L = 8$ case converges the slowest though its $\eta^{adapt}$ is higher for much of the training period. Our results suggest that the rate dependence on $L$ in Theorem \ref{thm:generalNN} is indeed fundamental, at least for the general class of models we consider. If a tighter convergence rate for standard neural networks were  shown to be \textit{independent} of $L$, it would be solely due to the structure of the specific projection operators.

 The exception to the above trend is $L = 2$, for which the loss decayed quickly early on but  plateaued above all other loss curves for the rest of training. We hypothesize that this is because the iterates are converging to a different stationary point. Investigating this further, we extend the training period from 100 epochs to 1,000 epochs for $L = 2$, with the results shown in Figure \ref{fig:fig3}. While the iterates in Figure \ref{fig:fig2} appear to converge to a global minima within the training period, the iterates in Figure \ref{fig:fig3} appear to be converging to a stationary point at infinity, supported by the unbounded growth of $\lVert W_t \rVert$ in \ref{fig:fig3}b even as the loss appears to have converged. This reinforces the importance of analyzing gradient descent on neural networks without assuming the iterates remain in a bounded set.

\begin{figure}[hb]
    \centering
    \includegraphics[width=0.95\linewidth]{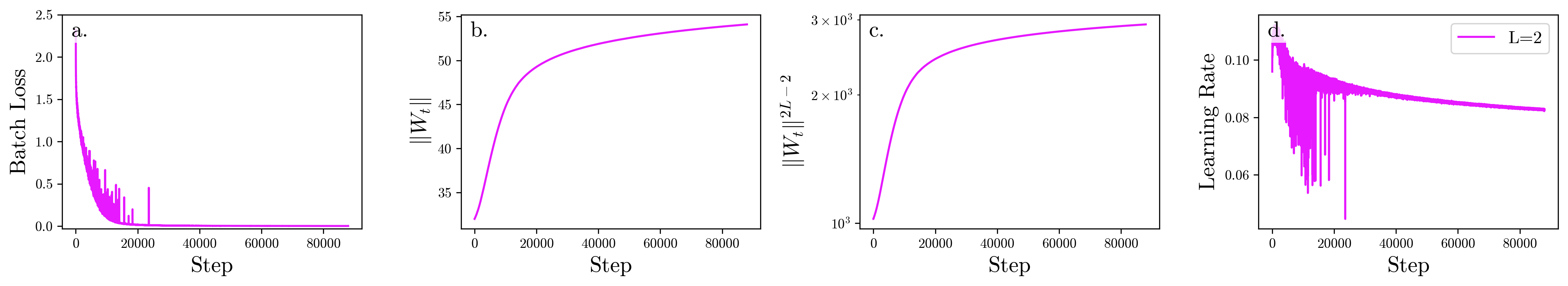}
    \caption{Extended training (1000 epochs) for $L = 2$ shared-weights neural network.}
    \label{fig:fig3}
    \vspace{-0.15in}
\end{figure}

\section{Conclusion}

We establish the first convergence rate of gradient descent on $L$-layer neural networks with no requirements on the depth, width, initialization, or dataset, proving that with an appropriate learning rate, the minimum squared gradient norm converges at rate $O(1/T^{1/L})$. Our analysis relies on precisely characterizing the double polynomial smoothness of neural network loss functions, which holds for a broad range of popular activation functions and loss functions. By accounting for the potentially unbounded growth of the parameters away from initialization,     
our work bypasses the restrictive overparameterization and lazy-training limits that fail to capture practical feature learning.

\section{AI Use Statement}

 In this work, we used generative AI tools to assist with developing conceptual frameworks, formulating mathematical claims, writing and checking of proofs and the refining of mathematical statements. We have not used generative AI tools for providing critical ingredients for proving mathematical claims, proposing or refining hypotheses, designing research experiments, cleaning or processing data, or interpreting results, and generating synthetic datasets, translation, and data analysis are not applicable to this work. Additionally, we used generative AI tools for brainstorming, searching existing literature, implementing methods and creating/editing software code. We have reviewed all AI-assisted work. We take responsibility for the final content of this work, including text, claims or artifacts produced with the aid of generative AI.

\bibliography{iclr2027_conference}
\bibliographystyle{iclr2027_conference}

\pagebreak

\appendix

\section{Additional Discussion}

Our results open up a number of questions regarding the behavior of gradient descent on neural networks. The most natural question is regarding the tightness of the bounds, and obtaining a lower bound on the convergence rate. In particular, if the $O(1/T^{1/L})$ bound is tight, then the convergence rate of gradient descent is slowed as the number of layers increases. It would also open up a performance gap that might explain why other optimization algorithms, like ADAM, outperform gradient descent on neural networks. In addition, gradient descent on Lipschitz smooth functions is famously \textit{dimension-free}, but our analysis includes polynomial dependence on the network width $d_{max}$ and exponential dependence on the number of layers $L$. Determining whether these dependencies can be removed is an important future direction. Finally, although we significantly relax the assumptions of prior work, this classical framework still heavily relies on some kind of smoothness---in plain words, that shrinking in parameter space continuously maps to  shrinking in the gradient. This leaves the behavior of gradient descent on ReLU neural networks largely unexplained.

\section{Further Comparison with \cite{wang2026convergencegradientdescentgeneral}}
\label{sec:wang}
The work \cite{wang2026convergencegradientdescentgeneral} proposes polynomial generalized smoothness, where for a function $\phi$  and weights $W_1, W_2$, there exists a polynomial $P(\cdot)$ such that
\begin{equation*}
    \lVert \nabla \phi(W_2) - \nabla \phi(W_1) \rVert \leq P(\max\{ \lVert W_1 \rVert, \lVert W_2 \rVert \}) \lVert W_2 - W_1 \rVert
\end{equation*}
They demonstrate this property is generally preserved under addition, multiplication, and composition. For neural networks, by repeatedly composing the activation functions, they show that there exists $P(\cdot)$ such that
\begin{equation}
\label{eq:wangsmooth}
    \lVert \nabla \mathcal{L}(W_2) - \nabla \mathcal{L}(W_1) \rVert \leq P(\max\{ \lVert W_1 \rVert, \lVert W_2 \rVert \}) \lVert W_2 - W_1 \rVert,
\end{equation}
which leads to the following descent lemma,
\begin{equation}
\label{eq:wangdescentlemma}
    \mathcal{L}(W_2) \leq \mathcal{L}(W_1) + \langle \nabla \mathcal{L}(W_1),  W_2 - W_1\rangle_F + \frac{1}{2} P(\max\{ \lVert W_1 \rVert, \lVert W_2 \rVert \}) \lVert W_2 - W_1 \rVert^2.
\end{equation}
\textit{They do not derive the specific form of} $P$ beyond demonstrating that the exponents are natural numbers with some dependence on $L$, and that the number of terms is also some unspecified polynomial of $L$ (see Appendix B.2 of \cite{wang2026convergencegradientdescentgeneral}). 

While the above properties share similarities with the conditions derived in our work, they have the following two limitations. First, in order to dynamically adjust the learning rate $\eta_t$ to the local smoothness  via (\ref{eq:wangdescentlemma}), one needs to set $\eta_t$ small enough with respect to $P(\max\{ \lVert W_t \rVert, \lVert W_{t+1} \rVert \})$ which depends implicitly on information at the $t+1$ step that is not readily available. Instead, \cite{wang2026convergencegradientdescentgeneral} uses tools from measure theory and dynamical systems to upper bound the Lipschitz constant $c$ within a ball of radius $R$ as
\begin{equation*}
    c = P(\lVert W^* \rVert + R + \max_{W \in B_{W^*(R)} }\lVert \nabla \mathcal{L}(W) \rVert),
\end{equation*}
where $W^*$ is a stationary point of the function and $R$ depends on the local dissipative condition of the function. This is a static bound over the whole ball that also depends on the maximum gradient norm on the whole ball. They pick a \textit{constant} learning rate, inversely proportional to $c$, to achieve \textit{asymptotic convergence} to a stationary point. The work also relies on a a very wide neural network where $d_{L - 1} > nd$, which never holds in practice,  to leverage a gradient dominance property. 

In contrast, although we show that the local smoothness \textit{does} depend on both $\lVert W_1 \rVert$ and $\lVert W_2 \rVert$, when deriving the descent lemma (Lemma \ref{lemma:descent}) we can cleanly integrate over the line segment between $W_1$ and $W_2$, leaving behind only terms dependent on $W_1$. We can therefore dynamically set $\eta_t$ in terms of $\lVert W_t \rVert$ and $\mathcal{L}(W_t)$, which are easily accessible. Moreover, $\eta_t$ does not depend on hard-to-compute values like the location of a stationary point or the maximum gradient norm over a ball.

The second limitation is that they do not derive the specific form of $P$ satisfied by an $L$-layer neural network. Indeed, the generic Lipschitz assumptions considered in both \cite{wang2026convergencegradientdescentgeneral} and our work lead to an explosion of terms that require careful tracking and bounding to achieve a simple and interpretable form. A key contribution of our analysis is decomposing the change in $ \nabla_W F(x_i ; W) $ and $ \nabla \mathcal{L}(W) $ into difference terms of the components, bounding those in terms of $\lVert W_2 - W_1 \rVert$, and combining the multiple sums cleanly to reduce the bound as much as possible. This leads to Lemma \ref{lemma:Lipschitzsmooth} which states (up to some constants)
\begin{equation*}
    \lVert \nabla \mathcal{L}(W_2) - \nabla \mathcal{L}(W_1) \rVert \leq (1 + \mathcal{L}(W_1)^{1/2})\tilde{P}(\lVert W_1 \rVert, \lVert W_2 \rVert) \lVert W_2 - W_1 \rVert,
\end{equation*}
where 
\begin{equation*}
    \tilde{P}(\lVert W_1 \rVert, \lVert W_2 \rVert) =    \sum_{j = 0}^{2L - 2}  \lVert W_1 \rVert^j  \sum_{i = 0}^{2L - j- 2} \lVert W_2 \rVert^{i}.
\end{equation*}
A critical feature of our analysis is that reducing the power of the leading term of $\tilde{P}$ tightens the upper bound on the convergence rate of gradient descent, because if $\tilde{P}$ is of degree $N$, $\lVert W_t \rVert^N$ grows at rate $T^{N/(N+2)}$. By simplifying the bound to a polynomial of degree $N = 2L - 2$, we  show that $\lVert W_t \rVert^{2L - 2}$ grows at rate $T^{\frac{L - 1}{L}}$, leading to an overall $O(\frac{1}{T^{1/L}})$ convergence rate. This represents the first convergence rate of gradient descent on neural networks without any requirements on width, depth, initialization, or dataset.
\section{Proofs}

\subsection{Helper Lemmas}

In this section, we include some basic or well-known theorems and identities required for the analysis.

\begin{lemma}
\label{lemma:elementwisesigma}
Let $v \in \mathbb{R}^d$ and let $\sigma: \mathbb{R} \to \mathbb{R}$ be a scalar function such that $|\sigma(x)| \leq c_0 + c_1|x|$ for all $x \in \mathbb{R}$, where $c_0, c_1 \geq 0$. If $\sigma(v)$ denotes the element-wise application of $\sigma$ to $v$, then the Euclidean ($L_2$) norm of $\sigma(v)$ satisfies the inequality:

$$ \lVert \sigma(v) \rVert_2 \leq c_0 \sqrt{d} + c_1 \lVert v \rVert_2 $$
\end{lemma}

\begin{proof}
We can show this using the triangle inequality.
    \begin{align*}
        \lVert \sigma(v) \rVert = \sqrt{\sum_{i = 1}^d \sigma(v_i)^2} \leq \sqrt{\sum_{i = 1}^d (c_0 + c_1 |v_i |)^2} \leq \sqrt{\sum_{i = 1}^d c_0^2} + \sqrt{\sum_{i = 1}^d c_1^2 |v_i |^2} \leq c_0 \sqrt{d} + c_1 \lVert v \rVert_2.
    \end{align*}
\end{proof}

\begin{theorem}
\label{thm:binomial}
    (Binomial theorem)
For any nonnegative integer $n$ and scalars $x, y$ we have
\begin{equation*}
    (x + y)^n = \sum_{k = 0}^n \binom{n}{k} x^k y^{n- k}.
\end{equation*}
    
\end{theorem}

\begin{lemma}
\label{hlemma:innerprod}
    For vectors $a \in \mathbb{R}^{d_1}$, $b \in \mathbb{R}^{d_2}$ and matrix $M \in \mathbb{R}^{d_1 \times d_2}$, we have
    \begin{equation*}
        \langle a, M b \rangle = \langle M^T a, b \rangle.
    \end{equation*}
    
\end{lemma}
\begin{proof}
\begin{align*}
    \langle a, M b \rangle & = b^T M^T a = \langle b, M^T a \rangle.
\end{align*}
\end{proof}

\begin{lemma}
\label{hlemma:odot}
    For vectors $a, b, c \in \mathbb{R}^d$, we have
    \begin{equation*}
        \langle a \odot b, c \rangle = \langle b, a \odot c \rangle.
    \end{equation*}
\end{lemma}

\begin{proof}
Denote $diag(a)$ as the diagonal $d \times d$ matrix with $a$ on the diagonal. Then
\begin{align*}
    a \odot b = & diag(a) b, \\
    \langle a \odot b, c \rangle = & \langle diag(a) b, c \rangle , \\
    = & \langle b, diag(a)^T c \rangle, \\
    = & \langle b, diag(a) c \rangle, \\
    = & \langle b, a \odot c \rangle.
\end{align*}
\end{proof}

\begin{lemma} (Lp norms)
\label{lemma:lpnorms}
    For a sequence of scalars $a_1,..,a_n$, we have for $p \geq 2$,
    \begin{equation*}
        \big( \sum_{i = 1}^n |a_i|^p \big)^{1/p} \leq \big( \sum_{i = 1}^n |a_i|^2 \big)^{1/2} \leq \sum_{i = 1}^n |a_i| \leq \sqrt{n} \big( \sum_{i = 1}^n |a_i|^2 \big)^{1/2}.
    \end{equation*}
\end{lemma}

\begin{lemma} 
\label{lemma:projection}
Let $E_{\ell}$ and $E_{\ell}^*$ be defined as in (\ref{eq:defineE}).
For all $W \in \mathbb{R}^{D_2 \times D_1}$ and $M \in \mathbb{R}^{d_\ell \times d_{\ell - 1}}$, we have
    \begin{align*}
        \lVert E_{\ell}[W] \rVert \leq \lVert W \rVert,\\
        \lVert E_{\ell}^*[M] \rVert \leq \lVert M \rVert.
    \end{align*}
\end{lemma}
\begin{proof}
    Since $E_{\ell}$ just extracts a $d_\ell \times d_{\ell - 1}$ subset of $W$, the first statement holds. Since $E^*_{\ell}$ produces a larger $D_2 \times D_1$ matrix containing $M$ and zero entries, the second statement holds.
\end{proof}

\begin{lemma} Suppose a scalar function $\sigma:\mathbb{R} \to \mathbb{R}$ is Lipschitz continuous with constant $c$, such that for any $w_1, w_2 \in \mathbb{R}$, we have
\begin{equation*}
    |\sigma(w_1) - \sigma(w_2) | \leq c |w_1 - w_2 |. 
\end{equation*}
Then the elementwise extension of $\sigma$ to matrices is also Lipschitz continuous with the same constant, such that for $W, W' \in \mathbb{R}^{d_1 \times d_2}$, we have
\begin{equation*}
    \lVert \sigma(W) - \sigma(W') \rVert_F \leq c \lVert W - W' \rVert_F.
\end{equation*}
\end{lemma}

\begin{proof}
    \begin{align*}
        \lVert \sigma(W) - \sigma(W') \rVert_F = & \sqrt{\sum_{i,j} (\sigma(W_{ij}) -\sigma(W'_{ij}))^2} \\
        \leq & \sqrt{\sum_{i,j} c^2 (W_{ij} -W'_{ij})^2}\\
        = & c \lVert W - W' \rVert_F.
    \end{align*}
\end{proof}

\begin{lemma}
\label{lemma:hademard}
Let $A, B \in \mathbb{R}^{m \times n}$ be real-valued matrices, and let $\sigma: \mathbb{R} \to \mathbb{R}$ be an element-wise function such that $|\sigma(x)| \leq c_0 + c_1|x|$ for constants $c_0, c_1 \geq 0$ and for all $x \in \mathbb{R}$. Then we have
\[
\|\sigma(A) \odot B\|_F \leq (c_0 + c_1 \| A \|_F) \|B\|_F
\]

\end{lemma}

\begin{proof}
We have
\begin{align*}
    \|\sigma(A) \odot B\|_F = \sqrt{\sum_{i=1}^m \sum_{j=1}^n |(\sigma(A) \odot B)_{ij}|^2} &= \sqrt{\sum_{i=1}^m \sum_{j=1}^n |\sigma(A_{ij}) B_{ij}|^2} \\
    \leq & \sqrt{\sum_{i=1}^m \sum_{j=1}^n ((c_0 + c_1 | A_{ij}| )|B_{ij}|)^2} \\
    \leq & \sqrt{\sum_{i=1}^m \sum_{j=1}^n (c_0 |B_{ij}| + c_1 | A_{ij}| |B_{ij}|)^2} \\
    \leq &  \sqrt{\sum_{i=1}^m \sum_{j=1}^n c_0^2 |B_{ij}|^2 } + \sqrt{\sum_{i=1}^m \sum_{j=1}^n c_1^2 |A_{ij}|^2 |B_{ij}|^2 } \\
    \leq & c_0 \lVert B \rVert + c_1 \lVert A \rVert \lVert B \rVert.
\end{align*}

\end{proof}

\subsection{Proof of Lemma \ref{lemma:LLsmooth}}

\begin{proof}
We have $\nabla \mathcal{L}(W) =  \frac{1}{n} \sum_{i = 1}^n J_i'(F(x_i;W))\cdot \nabla_W F(x_i ; W)$. The change in gradient is as follows,
\begin{align*}
    \lVert \nabla \mathcal{L}(W_2) - \nabla \mathcal{L}(W_1) \rVert = & \lVert \frac{1}{n} \sum_{i = 1}^n J_i'(F(x_i;W_2)) \nabla_W F(x_i ; W_2) - \frac{1}{n} \sum_{i = 1}^n J_i'(F(x_i;W_1)) \nabla_W F(x_i ; W_1) \rVert \\
    = & \lVert \frac{1}{n} \sum_{i = 1}^n \Big[ J_i'(F(x_i;W_2)) \nabla_W F(x_i ; W_2) - J_i'(F(x_i;W_1)) \nabla_W F(x_i ; W_2) \\
     &+ J_i'(F(x_i;W_1)) \nabla_W F(x_i ; W_2) - J_i'(F(x_i;W_1)) \nabla_W F(x_i ; W_1)  \Big] \rVert \\
     = & \lVert \frac{1}{n} \sum_{i = 1}^n \Big[ (J_i'(F(x_i;W_2)) - J_i'(F(x_i;W_1))) \nabla_W F(x_i ; W_2) \\
    &+ J_i'(F(x_i;W_1)) (\nabla_W F(x_i ; W_2) - \nabla_W F(x_i ; W_1))  \Big] \rVert \\
    \leq & \lVert \frac{1}{n} \sum_{i = 1}^n  (J_i'(F(x_i;W_2)) - J_i'(F(x_i;W_1))) \nabla_W F(x_i ; W_2) \rVert \\
    &+ \lVert \frac{1}{n} \sum_{i = 1}^n J_i'(F(x_i;W_1)) (\nabla_W F(x_i ; W_2) - \nabla_W F(x_i ; W_1)) \rVert \\
    \overset{\text{Assumption \ref{assump:lipJ}}}{\leq} & \frac{1}{n} \sum_{i = 1}^n  c_J |F(x_i ; W_2) - F(x_i ; W_1) |  \lVert\nabla_W F(x_i ; W_2) \rVert \\
    &+ \frac{1}{n} \sum_{i = 1}^n |J_i'(F(x_i;W_1))| \lVert \nabla_W F(x_i ; W_2) - \nabla_W F(x_i ; W_1) \rVert .
\end{align*}
\end{proof}

\subsection{Proof of Lemma \ref{lemma:gradformula}}

\begin{proof}
For an element-wise function $\sigma(X)$, we have
\begin{equation*}
    d \sigma(X) = \sigma'(X) \odot dX.
\end{equation*}
So we have the following differentials.
\begin{align*}
    dh_{\ell} = &  \sigma'_\ell(z_{\ell}(W)) \odot d z_\ell\\
    d z_\ell = & d(E_\ell[W]) h_{\ell - 1} + E_\ell [W] d h_{\ell - 1} \\
    = & E_{\ell}[dW] h_{\ell - 1} + E_\ell [W] d h_{\ell - 1} \\
    d z_1 = & E_1[dW] x_i
\end{align*}

We have $\delta_L(W) = \sigma'_L(z_L(W))$, and $\delta_{\ell - 1}(W) = \sigma'_{\ell - 1}(z_{\ell - 1}(W)) \odot (E_{\ell}  [W]^T \delta_{\ell}(W))$. Then 
\begin{equation*}
    dF = \langle \delta_L, dz_L \rangle,
\end{equation*}
and for $\ell = 2,.., L$,
\begin{align*}
    \langle \delta_\ell(W), d z_\ell \rangle = & \langle \delta_\ell(W), E_\ell[dW] h_{\ell - 1} \rangle + \langle \delta_\ell(W), E_{\ell} [W] d h_{\ell - 1}\rangle, \\
     = & \langle \delta_\ell(W), E_\ell[dW] h_{\ell - 1} \rangle + \langle \delta_\ell(W), E_{\ell} [W] (\sigma'_{\ell - 1}(z_{\ell - 1}(W)) \odot d z_{\ell - 1})\rangle, \\
    \overset{\text{Lemma \ref{hlemma:innerprod}}}{=} & \langle \delta_\ell(W), E_\ell[dW] h_{\ell - 1} \rangle + \langle E_{\ell} [W]^T \delta_\ell(W),  \sigma'_{\ell - 1}(z_{\ell - 1}(W)) \odot d z_{\ell - 1} \rangle \\
    \overset{\text{Lemma \ref{hlemma:odot}}}{=} & \langle \delta_\ell(W), E_\ell[dW] h_{\ell - 1} \rangle + \langle  \sigma'_{\ell - 1}(z_{\ell - 1}(W)) \odot (E_{\ell} [W]^T \delta_\ell(W)),   d z_{\ell - 1}(W) \rangle \\
    = & \langle \delta_\ell(W), E_\ell[dW] h_{\ell - 1} \rangle + \langle  \delta_{\ell - 1}(W),   d z_{\ell - 1} \rangle .
\end{align*}
For $\ell = 1$ we have
\begin{align*}
    \langle \delta_1(W), d z_1 \rangle = \langle \delta_1(W), E_1[dW] h_0(W) \rangle, 
\end{align*}
which yields the following equation for $dF$:
\begin{align*}
    dF = & \sum_{\ell = 1}^{L} \langle \delta_\ell(W), E_\ell[dW] h_{\ell - 1} (W)\rangle ,\\
    = & \langle \sum_{\ell = 1}^{L} E^*_{\ell}[\delta_{\ell} (W) h_{\ell - 1}(W)^T], dW \rangle.
\end{align*}
We therefore have the following formula for the gradient $\nabla_W F (x_i ; W)$
\begin{equation*}
    \nabla_W F(x_i ; W) = \sum_{\ell = 1}^{L}  E^*_{\ell}[\delta_{\ell} (W) h_{\ell - 1}(W)^T].
\end{equation*}
\end{proof}

\subsection{Proof of Lemma \ref{lemma:lip_f}}

\begin{proof}
    \begin{align*}
        \lVert F(x_i ; W_2 ) - F(x_i ; W_1) \rVert = & \lVert \sigma_L(z_{L}(W_2)) - \sigma_L(z_L(W_1)) \rVert, \\
        \leq & c_2 \lVert z_L(W_2) - z_L(W_1) \rVert, \\
        \overset{\text{Lemma \ref{lemma:zchange}}}{\leq} &c_x d_{max}^{1/2} c_{max}^{L} \lVert W_2 - W_1 \rVert \sum_{j = 0}^{L - 1} \sum_{k = 0}^{L - j - 1}  \lVert W_1 \rVert^j \lVert W_2 \rVert^{k}, \\
        = & c_x d_{max}^{1/2} c_{max}^{L}  \lVert W_2 - W_1 \rVert \sum_{k = 0}^{L - 1} \sum_{j = 0}^{L - k - 1}  \lVert W_1 \rVert^j \lVert W_2 \rVert^{k}.
    \end{align*}
\end{proof}

\subsection{Proof of Lemma \ref{lemma:lip_nablaf}}

\begin{proof}

We decompose the change in $\nabla_W F(x_i ; W)$ in terms of the change in $\delta$ and $h$ as follows
\begin{align*}
    \lVert \nabla_W F(x_i ; W_2) - \nabla_W F(x_i ; W_1) \rVert \leq & \sum_{\ell = 1}^L \lVert E_\ell^*[\delta_\ell(W_2) h_{\ell - 1}(W_2)^T - \delta_\ell(W_1) h_{\ell - 1}(W_1)^T ] \rVert, \\
    \leq & \sum_{\ell = 1}^L \lVert \delta_\ell(W_2) h_{\ell - 1}(W_2)^T - \delta_\ell(W_1) h_{\ell - 1}(W_1)^T \rVert, \\
    \leq & \sum_{\ell = 1}^L [ \lVert \delta_\ell(W_2) \rVert  \lVert h_{\ell - 1}(W_2) - h_{\ell - 1}(W_1) \rVert  + \lVert h_{\ell - 1}(W_1) \rVert \lVert \delta_\ell(W_2) - \delta_\ell(W_1) \rVert].
\end{align*}

We break up the sum as follows because $\lVert h_0(W_2) - h_0(W_1) \rVert = \lVert x_i - x_i \rVert = 0$, and to handle separately the bound on $\lVert \delta_L(W_2) - \delta_L(W_1) \rVert$,
\begin{align*}
    \lVert \nabla_W F(x_i ; W_2) - \nabla_W F(x_i ; W_1) \rVert \leq & \sum_{\ell = 2}^L \lVert \delta_\ell(W_2) \rVert  \lVert h_{\ell - 1}(W_2) - h_{\ell - 1}(W_1) \rVert  + \sum_{\ell = 1}^L \lVert h_{\ell - 1}(W_1) \rVert \lVert \delta_\ell(W_2) - \delta_\ell(W_1) \rVert,\\
    \overset{(\ref{eq:l2})}{\leq} & c_2 \sum_{\ell = 2}^L\lVert \delta_\ell(W_2) \rVert  \lVert z_{\ell - 1}(W_2) - z_{\ell - 1}(W_1) \rVert  + \sum_{\ell = 1}^L  \lVert h_{\ell - 1}(W_1) \rVert \lVert \delta_\ell(W_2) - \delta_\ell(W_1) \rVert,\\
    = & c_2 \sum_{\ell = 2}^L\lVert \delta_\ell(W_2) \rVert  \lVert z_{\ell - 1}(W_2) - z_{\ell - 1}(W_1) \rVert + \sum_{\ell = 1}^{L - 1}  \lVert h_{\ell - 1}(W_1) \rVert \lVert \delta_\ell(W_2) - \delta_\ell(W_1) \rVert, \\
    &+\lVert h_{L - 1}(W_1) \rVert \lVert \delta_L(W_2) - \delta_L(W_1) \rVert.
\end{align*}

We bound these three terms. For the first term, we have the sum is bounded as follows,
\begin{align*}
    c_2 \sum_{\ell = 2}^L\lVert \delta_\ell(W_2) \rVert  \lVert z_{\ell - 1}(W_2) - z_{\ell - 1}(W_1) \rVert \overset{\text{Lemma \ref{lemma:deltabound}}}{\leq} &  c_{max}^{L} \sum_{\ell = 2}^L\lVert W_2 \rVert^{L - \ell}  \lVert z_{\ell - 1}(W_2) - z_{\ell - 1}(W_1) \rVert,\\
    \overset{\text{Lemma \ref{lemma:zchange}}}{\leq} & c_{max}^{2L} \sum_{\ell = 2}^L\lVert W_2 \rVert^{L - \ell}  c_x d_{max}^{1/2} \lVert W_2 - W_1 \rVert \sum_{k = 0}^{\ell - 2} \sum_{j = 0}^{\ell - k - 2}  \lVert W_1 \rVert^j \lVert W_2 \rVert^{k} ,\\
    = & c_x d_{max}^{1/2} c_{max}^{2L} \lVert W_2 - W_1 \rVert \sum_{\ell = 2}^L  \sum_{k = 0}^{\ell - 2} \lVert W_2 \rVert^{L - \ell + k} \sum_{j = 0}^{\ell - k - 2}  \lVert W_1 \rVert^j  ,\\
    = & c_x d_{max}^{1/2} c_{max}^{2L} \lVert W_2 - W_1 \rVert \sum_{\ell = 2}^L  \sum_{k = L - \ell}^{L - 2} \lVert W_2 \rVert^{k} \sum_{j = 0}^{L - k - 2}  \lVert W_1 \rVert^j  ,\\
    \leq & c_x d_{max}^{1/2} c_{max}^{2L} \lVert W_2 - W_1 \rVert  L \sum_{k = 0}^{L - 2}   \lVert W_2 \rVert^{k} \sum_{j = 0}^{L - k - 2}  \lVert W_1 \rVert^j  .
\end{align*}
We can bound the third term as 
\begin{align*}
    \lVert h_{L - 1}(W_1) \rVert \lVert \delta_L(W_2) - \delta_L(W_1) \rVert  \overset{\text{Lemma \ref{lemma:hbound} }}{\leq} & c_x d_{max}^{1/2} c_{max}^{L} \Big(\sum_{m = 0}^{L - 1} 
    \lVert W_1 \rVert^m \Big) \lVert \delta_L(W_2) - \delta_L(W_1) \rVert \\
    \overset{\text{Lemma \ref{lemma:delta_change}}}{\leq} & c_x^2 d_{max} c_{max}^{2L} \lVert W_2 - W_1 \rVert \Big(\sum_{m = 0}^{L - 1} 
    \lVert W_1 \rVert^m \Big) \sum_{k = 0}^{L - 1} \sum_{j = 0}^{L - k - 1}  \lVert W_1 \rVert^j \lVert W_2 \rVert^{k},\\
    = & c_x^2 d_{max} c_{max}^{2L} \lVert W_2 - W_1 \rVert  
     \sum_{k = 0}^{L - 1} \lVert W_2 \rVert^{k} \sum_{m = 0}^{L - 1} \sum_{j = 0}^{L - k - 1}  \lVert W_1 \rVert^{j+m} ,\\
     = & c_x^2 d_{max} c_{max}^{2L} \lVert W_2 - W_1 \rVert  
     \sum_{k = 0}^{L - 1} \lVert W_2 \rVert^{k} \sum_{m = 0}^{L - 1} \sum_{j = m}^{L - k - 1 + m}  \lVert W_1 \rVert^{j} ,\\
     \leq & L c_x^2 d_{max} c_{max}^{2L} \lVert W_2 - W_1 \rVert  
     \sum_{k = 0}^{L - 1} \lVert W_2 \rVert^{k}  \sum_{j = 0}^{2L - 2- k }  \lVert W_1 \rVert^{j}. 
\end{align*}
For the second term, we have the sum is bounded as
\begin{align*}
    \sum_{\ell = 1}^{L - 1}  \lVert h_{\ell - 1}(W_1) \rVert \lVert \delta_\ell(W_2) - &\delta_\ell(W_1) \rVert \overset{\text{Lemma \ref{lemma:hbound}}}{\leq} c_{max}^{L} \sum_{\ell = 1}^{L - 1}  c_x d_{max}^{1/2} \Big(\sum_{m = 0}^{\ell - 1} 
    \lVert W_1 \rVert^m \Big) \lVert \delta_\ell(W_2) - \delta_\ell(W_1) \rVert \\
    \overset{\text{Lemma \ref{lemma:delta_change}}}{\leq} & c_{max}^{L} \sum_{\ell = 1}^{L - 1}  c_x d_{max}^{1/2} \Big(\sum_{m = 0}^{\ell - 1} 
    \lVert W_1 \rVert^m \Big) \Bigg(c_x d_{max}^{1/2} c_{max}^{2L} \lVert W_2 - W_1 \rVert \sum_{k = \ell}^{L}  \sum_{i = k - \ell}^{2k - 1 - \ell}  \lVert W_2 \rVert^i \sum_{j = L - k}^{L - i + k - \ell - 1} \lVert W_1 \rVert^j \\
    &+ \sum_{k = 0}^{L-\ell - 1} \lVert W_2 \rVert^{k}  \lVert W_1 \rVert^{L - k - \ell - 1} \lVert W_2 - W_1 \rVert \Bigg)\\
    = & c_x^2 d_{max} c_{max}^{3L} \lVert W_2 - W_1 \rVert \sum_{\ell = 1}^{L-1}   \Big(\sum_{m = 0}^{\ell - 1} 
    \lVert W_1 \rVert^m\Big) \Bigg( \sum_{k = \ell}^{L}  \sum_{i = k - \ell}^{2k - 1 - \ell}  \lVert W_2 \rVert^i \sum_{j = L - k}^{L - i + k - \ell - 1} \lVert W_1 \rVert^j \\
    &+ \sum_{k = 0}^{L-\ell - 1} \lVert W_2 \rVert^{k}  \lVert W_1 \rVert^{L - k - \ell - 1}   \Bigg)\\
    = & c_x^2 d_{max} c_{max}^{3L} \lVert W_2 - W_1 \rVert \Bigg[\sum_{\ell = 1}^{L-1}   \Big(\sum_{m = 0}^{\ell - 1} 
    \lVert W_1 \rVert^m\Big) \Bigg( \sum_{k = \ell}^{L}  \sum_{i = k - \ell}^{2k - 1 - \ell}  \lVert W_2 \rVert^i \sum_{j = L - k}^{L - i + k - \ell - 1} \lVert W_1 \rVert^j \Bigg)\\
    &+ \sum_{\ell = 1}^{L - 1} \Big(\sum_{m = 0}^{\ell - 1} 
    \lVert W_1 \rVert^m\Big) \sum_{k = 0}^{L-\ell - 1} \lVert W_2 \rVert^{k}  \lVert W_1 \rVert^{L - k - \ell - 1}  \Bigg]
\end{align*}

We have the quintuple sum can be bounded and simplified as follows
\begin{align*}
    \sum_{\ell = 1}^{L-1}   \Big(\sum_{m = 0}^{\ell - 1} 
    \lVert W_1 \rVert^m\Big) \Bigg( \sum_{k = \ell}^{L}  \sum_{i = k - \ell}^{2k - 1 - \ell}  \lVert W_2 \rVert^i \sum_{j = L - k}^{L - i + k - \ell - 1} \lVert W_1 \rVert^j \Bigg) = & \sum_{\ell = 1}^{L-1}   
    \sum_{k = \ell}^{L}  \sum_{i = k - \ell}^{2k - 1 - \ell}  \lVert W_2 \rVert^i \sum_{m = 0}^{\ell - 1}  \sum_{j = L - k}^{L - i + k - \ell - 1} \lVert W_1 \rVert^{m + j},\\
    = & \sum_{\ell = 1}^{L-1}   
    \sum_{k = \ell}^{L}  \sum_{i = k - \ell}^{2k - 1 - \ell}  \lVert W_2 \rVert^i \sum_{m = 0}^{\ell - 1}  \sum_{j = L - k + m}^{L - i + k - \ell - 1 + m} \lVert W_1 \rVert^{j},\\
    \leq & \sum_{\ell = 1}^{L-1}   
    \sum_{k = \ell}^{L}  \sum_{i = k - \ell}^{2k - 1 - \ell}  \lVert W_2 \rVert^i \sum_{m = 0}^{\ell - 1}  \sum_{j = L - k}^{L - i + k - 2} \lVert W_1 \rVert^{j},\\
    \leq & L \sum_{\ell = 1}^{L-1}   
    \sum_{k = \ell}^{L}  \sum_{i = k - \ell}^{2k - 1 - \ell}  \lVert W_2 \rVert^i  \sum_{j = L - k}^{L - i + k - 2} \lVert W_1 \rVert^{j},\\
    \leq & L \sum_{\ell = 1}^{L-1}   
    \sum_{k = \ell}^{L}  \sum_{i = k - \ell}^{2k - 1 - \ell}  \lVert W_2 \rVert^i  \sum_{j = 0}^{2L - i- 2} \lVert W_1 \rVert^{j},\\
    \leq & L^2 \sum_{\ell = 1}^{L-1}   \sum_{i = 0}^{2L - 1 - \ell}  \lVert W_2 \rVert^i  \sum_{j = 0}^{2L - i- 2} \lVert W_1 \rVert^{j},\\
    \leq & L^3  \sum_{i = 0}^{2L - 2}  \lVert W_2 \rVert^i  \sum_{j = 0}^{2L - i- 2} \lVert W_1 \rVert^{j}.
\end{align*}

And the triple sum can be bounded as
\begin{align*}
    \sum_{\ell = 1}^{L - 1} \Big(\sum_{m = 0}^{\ell - 1} 
    \lVert W_1 \rVert^m\Big) \sum_{k = 0}^{L-\ell - 1} \lVert W_2 \rVert^{k}  \lVert W_1 \rVert^{L - k - \ell - 1} = & \sum_{\ell = 1}^{L - 1}
 \sum_{k = 0}^{L-\ell - 1} \lVert W_2 \rVert^{k}   \sum_{m = 0}^{\ell - 1} \lVert W_1 \rVert^{L - k - \ell - 1 + m} \\
 = & \sum_{\ell = 1}^{L - 1}
 \sum_{k = 0}^{L-\ell - 1} \lVert W_2 \rVert^{k}   \sum_{m = L - k - \ell - 1}^{L - k - 2} \lVert W_1 \rVert^{m} \\
 \leq & \sum_{\ell = 1}^{L - 1}
 \sum_{k = 0}^{L-\ell - 1} \lVert W_2 \rVert^{k}   \sum_{m = 0}^{L - k - 2} \lVert W_1 \rVert^{m} \\
 \leq & L  \sum_{k = 0}^{L- 2} \lVert W_2 \rVert^{k}   \sum_{m = 0}^{L - k - 2} \lVert W_1 \rVert^{m}.
\end{align*}

Putting it all together and using the fact that $L \geq 1$, $d_{max} \geq 1$, $c_x \geq 1$, results in 
\begin{align*}
    \lVert \nabla_W F(x_i ; W_2) &- \nabla_W F(x_i ; W_1) \rVert \\
    \leq& L^3 c_x^2 d_{max} c_{max}^{3L} \lVert W_2 - W_1 \rVert \Bigg( \sum_{k = 0}^{L - 2}   \lVert W_2 \rVert^{k} \sum_{j = 0}^{L - k - 2}  \lVert W_1 \rVert^j \\
    &+ 
     \sum_{k = 0}^{L - 1} \lVert W_2 \rVert^{k}  \sum_{j = 0}^{2L - 2- k }   \lVert W_1 \rVert^{j} +  \sum_{i = 0}^{2L - 2}  \lVert W_2 \rVert^i  \sum_{j = 0}^{2L - i- 2} \lVert W_1 \rVert^{j}+   \sum_{k = 0}^{L- 2} \lVert W_2 \rVert^{k}   \sum_{m = 0}^{L - k - 2} \lVert W_1 \rVert^{m} \Bigg) \\
    \leq &  4 c_x^2 d_{max} L^3 c_{max}^{3L} \lVert W_2 - W_1 \rVert \sum_{k = 0}^{2L - 2}  \lVert W_2 \rVert^k  \sum_{j = 0}^{2L - k- 2} \lVert W_1 \rVert^{j}, \\
    = & 4 c_x^2 d_{max} L^3 c_{max}^{3L} \lVert W_2 - W_1 \rVert \sum_{j = 0}^{2L - 2}  \lVert W_1 \rVert^j  \sum_{k = 0}^{2L - j- 2} \lVert W_2 \rVert^{k}. 
\end{align*}
where in the last step we exchange the summations.
    
\end{proof}

\subsection{Proof of Lemma \ref{lemma:Lipschitzsmooth}}
\begin{proof}

We decompose the change in gradient $\lVert \nabla \mathcal{L}(W_2) - \nabla \mathcal{L}(W_1) \rVert$ in terms of $|F(x_i ; W_2) - F(x_i ; W_1) |$ and $\lVert \nabla_W F(x_i ; W_2) - \nabla_W F(x_i ; W_1) \rVert$ via Lemma \ref{lemma:LLsmooth}. In order to leverage Lemma \ref{lemma:LLsmooth}, we require the following bound on the gradient norm.
\begin{align*}
    \lVert \nabla_W F(x_i ; W) \rVert \leq & \sum_{\ell = 1}^{L}  \lVert \delta_{\ell} (W) h_{\ell - 1}(W)^T \rVert, \\
    \leq & \sum_{\ell = 1}^{L} \lVert \delta_{\ell} (W) \rVert \lVert h_{\ell - 1}(W) \rVert, \\
    \overset{\text{Lemmas \ref{lemma:deltabound} and \ref{lemma:hbound}}}{\leq} & c_{max}^{L} \sum_{\ell = 1}^{L} \lVert W \rVert^{L - \ell} c_x d_{max}^{1/2} \sum_{k = 0}^{\ell - 1} \lVert W \rVert^k, \\
    = & c_x d_{max}^{1/2} c_{max}^{L} \sum_{\ell = 1}^{L} \sum_{k = 0}^{\ell - 1} \lVert W \rVert^{L - \ell + k},\\
    = & c_x d_{max}^{1/2} c_{max}^{L} \sum_{\ell = 1}^{L} \sum_{k = L - \ell}^{L - 1} \lVert W \rVert^{k}.
\end{align*}
So we have the bound
\begin{align}
\label{eq:gradfbound}
    \lVert \nabla_W F(x_i ; W) \rVert \leq c_x d_{max}^{1/2} c_{max}^{L} \sum_{\ell = 1}^{L} \sum_{k = 0}^{L - 1} \lVert W \rVert^{k} = c_x d_{max}^{1/2} c_{max}^{L} L \sum_{k = 0}^{L - 1} \lVert W \rVert^{k}.
    \end{align}
Then we can use Lemmas \ref{lemma:lip_f} and \ref{lemma:lip_nablaf} to achieve an upper bound in terms of $\lVert W_2 - W_1 \rVert$, $\lVert W_1 \rVert$, and $\lVert W_2 \rVert$.
\begin{align*}
    \lVert \nabla \mathcal{L}(W_2) - \nabla \mathcal{L}(W_1) \rVert 
     \overset{\text{Lemma \ref{lemma:LLsmooth}}}{\leq} &  \frac{1}{n} \sum_{i = 1}^n  c_J |F(x_i ; W_2) - F(x_i ; W_1) |  \lVert\nabla_W F(x_i ; W_2) \rVert \\
    &+ \frac{1}{n} \sum_{i = 1}^n |J_i'(F(x_i;W_1))| \lVert \nabla_W F(x_i ; W_2) - \nabla_W F(x_i ; W_1) \rVert 
     \\
    \overset{\text{Lemma \ref{lemma:lip_f} and (\ref{eq:gradfbound})}}{\leq} &  c_J c_x d_{max}^{1/2} c_{max}^L \lVert W_2 - W_1 \rVert \sum_{j = 0}^{L - 1} \sum_{k = 0}^{L - j - 1}  \lVert W_1 \rVert^j \lVert W_2 \rVert^{k}  \cdot c_x d_{max}^{1/2}c_{max}^L L \sum_{q = 0}^{L - 1} \lVert W_2 \rVert^{q}\\
    &+ \frac{1}{n} \sum_{i = 1}^n |J_i'(F(x_i;W_1))| \lVert \nabla_W F(x_i ; W_2) - \nabla_W F(x_i ; W_1) \rVert \\
    = &  c_J c_x^2 d_{max} c_{max}^{2L} L \lVert W_2 - W_1 \rVert \sum_{j = 0}^{L - 1} \lVert W_1 \rVert^j \sum_{k = 0}^{L - j - 1}  \sum_{q = k}^{L - 1 + k} \lVert W_2 \rVert^{q}\\
    &+ \frac{1}{n} \sum_{i = 1}^n |J_i'(F(x_i;W_1))| \lVert \nabla_W F(x_i ; W_2) - \nabla_W F(x_i ; W_1) \rVert \\
    \overset{\text{Lemma \ref{lemma:lip_nablaf}}}{\leq} &  c_J c_x^2 d_{max}c_{max}^{2L}  L \lVert W_2 - W_1 \rVert \sum_{j = 0}^{L - 1} \lVert W_1 \rVert^j \sum_{k = 0}^{L - j - 1}  \sum_{q = 0}^{2L - 2 - j} \lVert W_2 \rVert^{q}\\
    &+ 4 c_x^2 d_{max} c_{max}^{3L} L^3 \lVert W_2 - W_1 \rVert \sum_{j = 0}^{2L - 2}  \lVert W_1 \rVert^j  \sum_{k = 0}^{2L - j- 2} \lVert W_2 \rVert^{k}\cdot \frac{1}{n} \sum_{i = 1}^n |J_i'(F(x_i;W_1))|.
    \end{align*}
    
    From Lipschitz smoothness of $J$, we have for all $W$ (See  Lemma 2.28 in \cite{garrigos2024handbookconvergencetheoremsstochastic}), 
\begin{equation}
\label{eq:boundongrad}
    | J'(F(x_i; W)) |^2 \leq  2c_J J_i(F(x_i; W)),
\end{equation} 
    \begin{align*}
    \frac{1}{n} \sum_{i = 1}^n|J_i'(F(x_i;W))| \leq& \frac{1}{n} \sum_{i = 1}^n (2 c_J J_i(F(x_i; W)))^{1/2} \\ 
    \leq & ( \frac{1}{n}\sum_{i = 1}^n 2 c_J J_i(F(x_i; W)))^{1/2}\\
    \leq & ( 2 c_J \mathcal{L}( W))^{1/2},
    \end{align*}
where we use Jensen's inequality and the concavity of the square-root function. So we have
    \begin{equation}
    \label{eq:boundloss}
        \frac{1}{n} \sum_{i = 1}^n |J_i'(F(x_i; W))| \leq ( 2 c_J \mathcal{L}(W))^{1/2}.
    \end{equation}
    Plugging this into the expression, we have
    \begin{align*}
    \lVert \nabla \mathcal{L}(W_2) - \nabla \mathcal{L}(W_1) \rVert \overset{(\ref{eq:boundloss})}{\leq} &  c_J c_x^2 d_{max} c_{max}^{2L} L^2 \lVert W_2 - W_1 \rVert \sum_{j = 0}^{L - 1} \lVert W_1 \rVert^j \sum_{k = 0}^{2L - 2 - j} \lVert W_2 \rVert^{k}\\
    &+ 4 (2 c_J \mathcal{L}(W_1))^{1/2} c_x^2 d_{max} c_{max}^{3L} L^3 \lVert W_2 - W_1 \rVert \sum_{j = 0}^{2L - 2}  \lVert W_1 \rVert^j  \sum_{k = 0}^{2L - j- 2} \lVert W_2 \rVert^{k},\\
    \overset{c_J, c_{max} \geq 1}{\leq} &  8 c_J c_x^2 d_{max} c_{max}^{3L} L^3  \lVert W_2 - W_1 \rVert \\
    &\cdot \Bigg[  \sum_{j = 0}^{L - 1} \lVert W_1 \rVert^j \sum_{k = 0}^{2L - 2 - j} \lVert W_2 \rVert^{k}+ \mathcal{L}(W_1)^{1/2}   \sum_{j = 0}^{2L - 2}  \lVert W_1 \rVert^j  \sum_{k = 0}^{2L - j- 2} \lVert W_2 \rVert^{k} \Bigg]\\
    \leq &  8 c_J c_x^2 d_{max} c_{max}^{3L} L^3  \lVert W_2 - W_1 \rVert (1 + \mathcal{L}(W_1)^{1/2})   \sum_{j = 0}^{2L - 2}  \lVert W_1 \rVert^j  \sum_{k = 0}^{2L - j- 2} \lVert W_2 \rVert^{k}.
\end{align*}
\end{proof}

\subsection{Proof of Lemma \ref{lemma:descent}}
\begin{proof}

We denote $v = W_2 - W_1$, and parametrize $\mathcal{L}$ between $W_1$ and $W_2$ such that for $t \in [0, 1]$, we have
\begin{align*}
    \phi(t) =& \mathcal{L}(W_1 + tv),\\
    \phi'(t) =& \langle \nabla \mathcal{L}(W_1 + tv),  v \rangle_F.  
\end{align*}

By the fundamental theorem of calculus, we have
\begin{align*}
    \mathcal{L}(W_2) - \mathcal{L}(W_1) &= \phi(1) - \phi(0) = \int_0^1 \phi'(t) dt,\\
    & = \int_0^1 \langle \nabla \mathcal{L}(W_1), v \rangle_F \;dt + \int_0^1 \langle \nabla \mathcal{L}(W_1 + tv) - \nabla \mathcal{L}(W_1), v \rangle_F \; dt ,\\
    & = \langle \nabla \mathcal{L}(W_1), v \rangle_F + \int_0^1 \langle \nabla \mathcal{L}(W_1 + tv) - \nabla \mathcal{L}(W_1), v\rangle_F \; dt,\\
    & \leq \langle \nabla \mathcal{L}(W_1), v \rangle_F + \int_0^1 \lVert \nabla \mathcal{L}(W_1 + tv) - \nabla \mathcal{L}(W_1) \rVert \lVert v\rVert dt, 
\end{align*} 
where in the last step we use the Cauchy-Schwarz inequality and the fact that integrals preserve inequalities. For the integral, we have by Lemma \ref{lemma:Lipschitzsmooth},
\begin{align*}
    \int_0^1 \lVert \nabla \mathcal{L}(W_1 + tv) - \nabla \mathcal{L}(W_1) \rVert \lVert v\rVert dt \leq & 8 c_J c_x^2 d_{max} c_{max}^{3L} L^3  (1 + \mathcal{L}(W_1)^{1/2})  \lVert v \rVert^2   \sum_{j = 0}^{2L - 2}  \lVert W_1 \rVert^j   \int_0^1  t \sum_{i = 0}^{2L - j- 2}   \lVert W_1 + tv \rVert^i    dt,
\end{align*}
where we pull out any terms not dependent on $t$ from the integral. For any nonnegative integer $N$, we have by the binomial expansion and reindexing
\begin{align*}
    \int_0^1 t \sum_{i = 0}^{N} \lVert W_1 + t v \rVert^i \, dt \leq & \int_0^1 t \sum_{i = 0}^{N} (\lVert W_1 \rVert  + \lVert t v \rVert)^i \, dt \\
    = & \int_0^1 t \sum_{i = 0}^{N} \sum_{k = 0}^{i} \binom{i}{k} \lVert W_1 \rVert^{i - k} t^k \lVert v \rVert^k \, dt \\
    = & \int_0^1 t \sum_{k = 0}^{N} \sum_{i = k}^{N} \binom{i}{k} \lVert W_1 \rVert^{i - k} t^k \lVert v \rVert^k \, dt \\
    = & \sum_{k = 0}^{N} \sum_{i = k}^{N} \binom{i}{k} \lVert W_1 \rVert^{i - k} \lVert v \rVert^k \int_0^1 t^{k + 1}  \, dt \\
    = & \sum_{k = 0}^{N} \sum_{i = k}^{N} \binom{i}{k} \lVert W_1 \rVert^{i - k} \lVert v \rVert^k \frac{1}{k + 2} \\
    \leq & 2^{N} \sum_{k = 0}^{N} 
    \frac{1}{k + 2} \lVert v \rVert^k   \sum_{i = k}^{N}  \lVert W_1 \rVert^{i - k} 
    =  2^{N} \sum_{k = 0}^{N} 
    \frac{1}{k + 2} \lVert v \rVert^k   \sum_{i = 0}^{N - k}  \lVert W_1 \rVert^{i},
\end{align*}
where in the last inequality we use the fact that $\binom{i}{k} \leq 2^{i} \leq 2^{N}$.
Therefore we have
\begin{align*}
    \sum_{j = 0}^{2L - 2}  \lVert W_1 \rVert^j   \int_0^1  t \sum_{i = 0}^{2L - j- 2}   \lVert W_1 + tv \rVert^i    dt \leq & \sum_{j = 0}^{2L - 2} 2^{2L - j- 2} \sum_{k = 0}^{2L - j- 2} 
    \frac{1}{k + 2} \lVert v \rVert^k   \sum_{i = 0}^{2L - j- 2 - k}  \lVert W_1 \rVert^{i + j} \\
    = & \sum_{j = 0}^{2L - 2} 2^{2L - j- 2} \sum_{k = 0}^{2L - j- 2} 
    \frac{1}{k + 2} \lVert v \rVert^k   \sum_{i = j}^{2L - 2 - k}  \lVert W_1 \rVert^{i } \\
    \leq & 2^{2L - 2} \sum_{j = 0}^{2L - 2} 
    \sum_{i = j}^{2L - 2 }  \lVert W_1 \rVert^{i }  \sum_{k = 0}^{2L - j- 2} 
    \frac{1}{k + 2} \lVert v \rVert^k   \\
    \leq & 2^{2L - 2} (2L - 1) 
    \sum_{i = 0}^{2L - 2 }  \lVert W_1 \rVert^{i }  \sum_{k = 0}^{2L - 2} 
    \frac{1}{k + 2} \lVert v \rVert^k ,
\end{align*}

and the integral term can be bounded as follows,
\begin{align*}
    \int_0^1  \lVert \nabla \mathcal{L}(W_1 &+ tv) - \nabla \mathcal{L}(W_1) \rVert \lVert v\rVert dt \\
    \leq & 8 c_J c_x^2 d_{max} c_{max}^{3L} L^3  (1 + \mathcal{L}(W_1)^{1/2})  \lVert v \rVert^2   \cdot 2^{2L - 2} (2L - 1) 
    \sum_{i = 0}^{2L - 2 }  \lVert W_1 \rVert^{i }  \sum_{k = 0}^{2L - 2} 
    \frac{1}{k + 2} \lVert v \rVert^k \\
    \leq & 2^{2L + 2} c_J c_x^2 d_{max} c_{max}^{3L} L^4  (1 + \mathcal{L}(W_1)^{1/2})  
    \sum_{i = 0}^{2L - 2 }  \lVert W_1 \rVert^{i }  \sum_{k = 0}^{2L - 2} 
    \frac{1}{k + 2} \lVert v \rVert^{k+2}.
\end{align*}
 Let $C =  2^{2L + 2} c_J c_x^2 d_{max} c_{max}^{3L} L^4 $. Then putting it all together gives
\begin{align*}
    \mathcal{L}(W_2) - \mathcal{L}(W_1) \leq &\langle \nabla \mathcal{L}(W_1), W_2 - W_1 \rangle_F \\
    &+ C \Big[ (1 + \mathcal{L}(W_1)^{1/2})  
    \sum_{i = 0}^{2L - 2 }  \lVert W_1 \rVert^{i } \Big]  \sum_{k = 0}^{2L - 2} 
    \frac{1}{k + 2} \lVert  W_2 - W_1 \rVert^{k+2}. 
\end{align*}    
\end{proof}

\subsection{Proof of Lemma \ref{lemma:onestep}}
Applying the descent lemma to $W_{t+1}$ and $W_t$ results in sums of powers of $\eta_t$. In the following, we set $\eta_t$ small enough by considering the coefficients of the $\eta_t^2$ terms. We then demonstrate that the terms corresponding to higher powers of $\eta_t$ are also controlled to be small, resulting in descent in function value.

For $W_{t+1} = W_t - \eta_t \nabla \mathcal{L}(W_t)$, we have by Lemma \ref{lemma:descent}
\begin{align*}
    \mathcal{L}(W_{t+1}) - \mathcal{L}(W_{t}) \leq &\langle \nabla \mathcal{L}(W_t), W_{t+1} - W_{t} \rangle_F + C (1 + \mathcal{L}(W_t)^{1/2})  
    \sum_{i = 0}^{2L - 2 }  \lVert W_t \rVert^{i }\sum_{k = 0}^{2L - 2} 
    \frac{1}{k + 2} \lVert  W_{t+1} - W_t \rVert^{k+2}  \\
    = & - \eta_t \lVert \nabla \mathcal{L}(W_t) \rVert^2 + C (1 + \mathcal{L}(W_t)^{1/2})  
    \sum_{i = 0}^{2L - 2 }  \lVert W_t \rVert^{i } \sum_{k = 0}^{2L - 2} 
    \frac{1}{k + 2} \lVert  \nabla \mathcal{L}(W_t) \rVert^{k+2} \eta_t^{k+2} .
\end{align*}
We observe the right hand side consists of $2L - 1$ nonnegative higher-order terms of the form $\eta_t^\gamma  \lVert  \nabla \mathcal{L}(W_t) \rVert^\gamma$, for $\gamma = 2,..., 2L$. 
We set 
\begin{equation*}
    \eta_t =  \frac{1}{2 L C (1 + \mathcal{L}(W_t)^{1/2})  
    \sum_{i = 0}^{2L - 2 }  \lVert W_t \rVert^{i } } = \frac{1}{\rho (1 + \mathcal{L}(W_t)^{1/2})  
    \sum_{i = 0}^{2L - 2 }  \lVert W_t \rVert^{i } }
\end{equation*}
where $\rho = 2  L C $. We want to demonstrate that $\eta_t$ is small enough such that each term is less than or equal to
    $\frac{\eta_t}{2L} \lVert  \nabla \mathcal{L}(W_t) \rVert^2$.
The $\gamma = 2$ case of $\eta_t^2 \lVert \nabla \mathcal{L}(W_t) \rVert^2 $ is trivial. We have
\begin{align*}
    \eta_t^2 \lVert  \nabla \mathcal{L}(W_t) \rVert^2 \frac{C}{2} (1 + \mathcal{L}(W_t)^{1/2})  
    \sum_{i = 0}^{2L - 2 }  \lVert W_t \rVert^{i } \leq &   \eta_t \cdot   \frac{ \frac{C}{2}  (1 + \mathcal{L}(W_t)^{1/2})  
    \sum_{i = 0}^{2L - 2 }  \lVert W_t \rVert^{i } }{2LC (1 + \mathcal{L}(W_t)^{1/2})  
    \sum_{i = 0}^{2L - 2 }  \lVert W_t \rVert^{i } }  \lVert  \nabla \mathcal{L}(W_t) \rVert^2 \\
    \leq &  \frac{\eta_t}{2L}\lVert  \nabla \mathcal{L}(W_t) \rVert^2. 
\end{align*}

Now for $k = 1,..., 2L - 2$, where $k + 2 = \gamma$, we want to show the following  statement,
\begin{align*}
    \eta_t^{k + 2} \lVert  \nabla \mathcal{L}(W_t) \rVert^{k + 2} \frac{C}{k + 2}  (1 + \mathcal{L}(W_t)^{1/2})  
    \sum_{i = 0}^{2L - 2 }  \lVert W_t \rVert^{i }  \leq & \frac{\eta_t}{2L} \lVert  \nabla \mathcal{L}(W_t) \rVert^2.
    \end{align*}
We have the gradient norm is bounded as
\begin{align*}
    \lVert \nabla \mathcal{L}(W_t) \rVert 
    \leq & \frac{1}{n} \sum_{i = 1}^n | J_i'(F(x_i;W))| \lVert \nabla_W F(x_i ; W) \rVert\\
    \overset{(\ref{eq:gradfbound})}{\leq} &\frac{1}{n} \sum_{i = 1}^n | J_i'(F(x_i;W))| c_x d_{max}^{1/2} c_{max}^L L \sum_{k = 0}^{L - 1} \lVert W \rVert^{k}\\
    \overset{(\ref{eq:boundloss})}{\leq} & ( 2 c_J \mathcal{L}(W))^{1/2} \cdot c_x d_{max}^{1/2} c_{max}^{L} L \sum_{k = 0}^{L - 1} \lVert W \rVert^{k},
\end{align*}
and we therefore have the bound
\begin{align}
\label{eq:etagrad}
\eta_t \lVert \nabla \mathcal{L}(W_t) \rVert \overset{(\ref{eq:gradfbound})}{\leq} & \frac{(2 c_J \mathcal{L}(W_t))^{1/2} c_x d^{1/2}_{max} c_{max}^{L} L \sum_{k = 0}^{L - 1} \lVert W_t \rVert^k}{2 L C (1 + \mathcal{L}(W_t)^{1/2})  
    \sum_{i = 0}^{2L - 2 }  \lVert W_t \rVert^{i } } \leq \frac{1}{2^{2L + 2} c_J^{1/2} c_x d_{max}^{1/2}c_{max}^{2L} L^4} \leq 1,
\end{align}
so we can bound powers of this term,
\begin{align*}
    \eta_t^{k + 2} \lVert  \nabla \mathcal{L}(W_t) & \rVert^{k + 2} \frac{C}{k + 2}  (1 + \mathcal{L}(W_t)^{1/2})  
    \sum_{i = 0}^{2L - 2 }  \lVert W_t \rVert^{i }  \\= & \Big(\eta_t \lVert \nabla \mathcal{L}(W_t) \rVert^{2}\Big) \Big(\eta_t  \frac{C}{k + 2}  (1 + \mathcal{L}(W_t)^{1/2})  
    \sum_{i = 0}^{2L - 2 }  \lVert W_t \rVert^{i } \Big) \eta_t^{k} \lVert  \nabla \mathcal{L}(W_t) \rVert^{k } \\
    \leq & \Big(\eta_t \lVert \nabla \mathcal{L}(W_t) \rVert^{2}\Big) \frac{1}{2L} \eta_t^{k} \lVert  \nabla \mathcal{L}(W_t) \rVert^{k } \\
    \overset{(\ref{eq:etagrad})}{\leq} & \Big(\eta_t \lVert \nabla \mathcal{L}(W_t) \rVert^{2}\Big) \frac{1}{2L}  \leq \frac{\eta_t}{2L} \lVert \nabla \mathcal{L}(W_t) \rVert^{2}.
\end{align*}

So we have shown that each of the $2L-1$ terms is less than or equal to
    $\frac{\eta_t}{2L} \lVert  \nabla \mathcal{L}(W_t) \rVert^2$. We have
\begin{align*}
    \mathcal{L}(W_{t+1}) - \mathcal{L}(W_t) \leq & - \eta_t \lVert \nabla \mathcal{L}(W_t) \rVert^2 + \eta_t \frac{2L-1}{2L} \lVert \nabla \mathcal{L}(W_t)\rVert^2 \\
    = & - \frac{\eta_t}{2L} \lVert \nabla \mathcal{L} (W_t) \rVert^2.
\end{align*}
Rearranging and summing on both sides yields 
\begin{align*}
    \sum_{t = 0}^{T - 1}\frac{\eta_t}{2L} \lVert \nabla \mathcal{L} (W_t) \rVert^2 \leq & \sum_{t = 0}^{T - 1}\mathcal{L}(W_{t}) - \mathcal{L}(W_{t+1}) \\
    \sum_{t = 0}^{T - 1} \eta_t \lVert \nabla \mathcal{L} (W_t) \rVert^2 \leq & 2L (\mathcal{L}(W_0) - \mathcal{L}(W_T)).
    \end{align*}
Since $\mathcal{L}(W_T) \geq 0$, we have the following bound, which we will leverage multiple times in the proof:
\begin{equation}
    \label{eq:cumulativeappend}
    \sum_{t = 0}^{T - 1} \eta_t \lVert \nabla \mathcal{L} (W_t) \rVert^2 \leq 2L \mathcal{L}(W_0).
\end{equation}
We can therefore bound the minimum squared gradient norm as follows.
    \begin{align*}
    \min_{t=0,...,T - 1} \lVert \nabla \mathcal{L} (W_t) \rVert^2 \sum_{t = 0}^{T - 1} \eta_t \leq & 2L \mathcal{L}(W_0), \\
    \min_{t=0,...,T - 1} \lVert \nabla \mathcal{L} (W_t) \rVert^2 \leq & \frac{2L \mathcal{L}(W_0)}{ \sum_{t = 0}^{T - 1} \eta_t}.
\end{align*}

\subsection{Proof of Theorem \ref{thm:generalNN}}
To show convergence of the gradient norm to zero, we need to show that the sum of learning rates $\sum_{t = 0}^{T - 1} \eta_t$ diverges. We achieve this by lower bounding with another series, which we show diverges. First, we note that since $\mathcal{L}(W_t)$ is nonincreasing with $t$, we have $(\mathcal{L}(W_t))^{1/2} \leq (\mathcal{L}(W_0))^{1/2}$. Moreover, we have for all $0 \leq j \leq N$,
\begin{equation}
    \label{eq:boundoneplus}
    \lVert W \rVert^j \leq \lVert W \rVert^N + 1.
\end{equation}
So we have for $N = 2L - 2$,
\begin{align*}
    \sum_{i = 0}^{2L - 2 }  \lVert W_t \rVert^{i } \leq &  
    (2L - 1)(1 + \lVert W_t \rVert^N).  
\end{align*}
 Let $Q_t = \rho (1 + \mathcal{L}(W_0)^{1/2})  
    (2L - 1)(1 + \lVert W_t \rVert^N) $. Then we have

\begin{align*}
    \eta_t = \frac{1}{\rho (1 + \mathcal{L}(W_t)^{1/2})  
    \sum_{i = 0}^{2L - 2 }  \lVert W_t \rVert^{i } } \geq \frac{1}{Q_t}
\end{align*}

To show the divergence of $\sum_{t = 0}^{\infty} \eta_t$, it suffices to show the divergence of $\sum_{t = 0}^{\infty} \frac{1}{Q_t}$.  We therefore want to track the growth of $\lVert W_t \rVert^N$ as $t$ increases.

We first observe that for any $ 0 \leq j \leq N \leq 2L - 2$, that 
\begin{align}
\label{eq:exponentbound}
    \lVert W_t \rVert^{j} \eta_t^{\frac{j}{N}} \leq & \lVert W_t \rVert^{j} \Big( \frac{1}{\rho \lVert W_t \rVert^{N}}\Big)^{\frac{j}{N}} \leq 1,
\end{align}
where we also use that $\rho > 1$. We can therefore bound for any $N \leq 2L - 2$,
\begin{align*}
    \lVert W_{t+1} \rVert^N = & \lVert W_t - \eta_t \nabla \mathcal{L}(W_t) \rVert^N \\
    \leq & (\lVert W_t \rVert + \eta_t \lVert \nabla \mathcal{L}(W_t) \rVert)^N \\
    = & \sum_{k = 0}^N \binom{N}{k} \lVert W_t \rVert^{N - k} \eta_t^k \lVert \nabla \mathcal{L}(W_t) \rVert^k \\
    = & \lVert W_t \rVert^N + N \lVert W_t \rVert^{N -1} \eta_t \lVert \nabla \mathcal{L}(W_t) \rVert + \sum_{k = 2}^N \binom{N}{k} \lVert W_t \rVert^{N - k} \eta_t^k \lVert \nabla \mathcal{L}(W_t) \rVert^k \\
    = & \lVert W_t \rVert^N + N \lVert W_t \rVert^{N -1} \eta_t^{\frac{N - 1}{N}} \eta_t^{\frac{1}{N}} \lVert \nabla \mathcal{L}(W_t) \rVert + \sum_{k = 2}^N \binom{N}{k} \lVert W_t \rVert^{N - k} \eta_t^{k/2} \eta_t^{k/2} \lVert \nabla \mathcal{L}(W_t) \rVert^k \\
    \overset{\eta_t \leq 1}{\leq} & \lVert W_t \rVert^N + N \lVert W_t \rVert^{N -1} \eta_t^{\frac{N - 1}{N}} \eta_t^{\frac{1}{N}} \lVert \nabla \mathcal{L}(W_t) \rVert + \sum_{k = 2}^N \binom{N}{k} \lVert W_t \rVert^{N - k} \eta_t \eta_t^{k/2} \lVert \nabla \mathcal{L}(W_t) \rVert^k \\
    \overset{(\ref{eq:exponentbound})}{\leq} & \lVert W_t \rVert^N + N \eta_t^{1/N} \lVert \nabla \mathcal{L}(W_t) \rVert + \sum_{k = 2}^N \binom{N}{k} \eta_t^{k/2} \lVert \nabla \mathcal{L}(W_t) \rVert^k \\
    \lVert W_{t+1} \rVert^N - \lVert W_{t} \rVert^N \leq & N \eta_t^{1/N}\lVert \nabla \mathcal{L}(W_t) \rVert + \sum_{k = 2}^N \binom{N}{k} \eta_t^{k/2} \lVert \nabla \mathcal{L}(W_t) \rVert^k 
\end{align*}
When we sum from $t = 0$ to $T-1$ on both sides, we achieve a telescoping sum as follows,
\begin{align*}
    \sum_{t = 0}^{T - 1}\lVert W_{t+1} \rVert^N - \lVert W_{t} \rVert^N \leq & N  \sum_{t = 0}^{T - 1} \eta_t^{1/N} \lVert \nabla \mathcal{L}(W_t) \rVert + \sum_{t = 0}^{T - 1} \sum_{k = 2}^N \binom{N}{k} \eta_t^{k/2} \lVert \nabla \mathcal{L}(W_t) \rVert^k, \\
    \lVert W_T \rVert^N \leq & \lVert W_0 \rVert^N +  N  \sum_{t = 0}^{T - 1} \eta_t^{1/N} \lVert \nabla \mathcal{L}(W_t) \rVert + \sum_{t = 0}^{T - 1} \sum_{k = 2}^N \binom{N}{k}  \eta_t^{k/2} \lVert \nabla \mathcal{L}(W_t) \rVert^k. 
\end{align*}
The double sum last term can be bounded via (\ref{eq:cumulativeappend}) as follows (see also, Lemma \ref{lemma:lpnorms})
\begin{align}
\label{eq:WtN}
    \lVert W_T \rVert^N \leq & \lVert W_0 \rVert^N + N \sum_{t = 0}^{T - 1} \eta_t^{1/N} \lVert \nabla \mathcal{L}(W_t) \rVert +\sum_{k = 2}^N \binom{N}{k} (2L \mathcal{L}(W_0))^{k/2}. 
\end{align}
The crux of this proof is bounding $\sum_{t = 0}^{T - 1} \eta_t^{1/N} \lVert \nabla \mathcal{L}(W_t) \rVert$. When $N = 2$, which is true in the two-layer case, this term can be directly bounded using (\ref{eq:cumulativeappend}), leading to an $O(\sqrt{T})$ growth rate. However, for $N > 2$, we need to control the cumulative terms more carefully. For $N = 2L - 2$, we can decompose the sum as follows:
\begin{align*}
    \sum_{t = 0}^{T - 1} \eta_t^{1/N} \lVert \nabla \mathcal{L}(W_t) \rVert &= \sum_{t = 0}^{T - 1} \eta_t^{1/N - 1/2} \eta_t^{1/2} \lVert \nabla \mathcal{L}(W_t) \rVert = \sum_{t = 0}^{T - 1} \eta_t^{\frac{2 - N}{2N}} \eta_t^{1/2} \lVert \nabla \mathcal{L}(W_t) \rVert, \\
    \overset{\text{Cauchy-Schwarz}}{\leq} & \Big( \sum_{t = 0}^{T - 1} \eta_t^\frac{2 - N}{N} \Big)^{1/2} \Big( \sum_{t = 0}^{T - 1} \eta_t \lVert \nabla \mathcal{L}(W_t) \rVert^2 \Big)^{1/2}\\
    \overset{(\ref{eq:cumulativeappend})}{\leq} & \Big( \sum_{t = 0}^{T - 1} \eta_t^\frac{2 - N}{N} \Big)^{1/2} (2L \mathcal{L}(W_0))^{1/2}.
\end{align*}
We have
\begin{align*}
    \eta_t^\frac{2 - N}{N} = & \Big(\frac{1}{\eta_t} \Big)^{\frac{N - 2}{N}} = (\rho (1 + \mathcal{L}(W_t)^{1/2})  
    \sum_{i = 0}^{2L - 2 }  \lVert W_t \rVert^{i })^{\frac{N - 2}{N}} \\
    \overset{\text{Lemma \ref{lemma:onestep}}}{\leq} & (\rho (1 + \mathcal{L}(W_0)^{1/2}) 
    \sum_{i = 0}^{2L - 2 }  \lVert W_t \rVert^{i })^{\frac{N - 2}{N}} \\
    \overset{(\ref{eq:boundoneplus})}{\leq} & (\rho (1 + \mathcal{L}(W_0)^{1/2}) 
    \sum_{i = 0}^{2L - 2 }  (1 + \lVert W_t \rVert^N))^{\frac{N - 2}{N}} \\
    = & (\rho (1 + \mathcal{L}(W_0)^{1/2}) 
    (2L - 1)  (1 + \lVert W_t \rVert^N))^{\frac{N - 2}{N}} \\
\end{align*}
Plugging this back into (\ref{eq:WtN}), we have
\begin{align*}
    \lVert W_T \rVert^N \leq & \lVert W_0 \rVert^N + N \Big( \sum_{t = 0}^{T - 1} (\rho (1 + \mathcal{L}(W_0)^{1/2}) 
    (2L - 1)  (1 + \lVert W_t \rVert^N))^{\frac{N - 2}{N}}  \Big)^{1/2} (2L \mathcal{L}(W_0))^{1/2}  \\
    &+\sum_{k = 2}^N \binom{N}{k} (2L \mathcal{L}(W_0))^{k/2},\\
    = & \lVert W_0 \rVert^N + N (\rho (1 + \mathcal{L}(W_0)^{1/2}) 
    (2L - 1) )^{\frac{N - 2}{2N}}\Big( \sum_{t = 0}^{T - 1}  (1 + \lVert W_t \rVert^N)^{\frac{N - 2}{N}}  \Big)^{1/2} (2L \mathcal{L}(W_0))^{1/2}  \\
    &+\sum_{k = 2}^N \binom{N}{k} (2L \mathcal{L}(W_0))^{k/2}.
\end{align*}
We therefore have that $\lVert W_T \rVert^N$ is bounded as
\begin{align*}
    \lVert W_T \rVert^N \leq & A \Big( \sum_{t = 0}^{T - 1}  (1 + \lVert W_t \rVert^N)^{\frac{N - 2}{N}}  \Big)^{1/2} + B. 
\end{align*}
where $A = N (\rho (1 + \mathcal{L}(W_0)^{1/2}) 
    (2L - 1) )^{\frac{N - 2}{2N}}(2L \mathcal{L}(W_0))^{1/2}$ and $B = \lVert W_0 \rVert^N + \sum_{k = 2}^N \binom{N}{k} (2L \mathcal{L}(W_0))^{k/2}$.

So the growth of $\lVert W_T \rVert^N$ depends on the growth of the $t = 0,..., T-1$ iterates before; for example, as long as those can be shown to grow at most linearly, then $\lVert W_T \rVert^N$ will also be linear in $T$. In particular, we can show the growth is \textit{sublinear} by proving with induction that for all $T > 0$, there exists $\alpha > 0$, $\beta > 0$ such that
\begin{align*}
    \lVert W_T \rVert^N \leq \alpha T^{\frac{N}{N+2}} + \beta.
\end{align*}
Let $\beta = B$.
For the base case, we have
\begin{align*}
    \lVert W_0 \rVert^N \leq B.
\end{align*}

For the inductive step, suppose the hypothesis holds for $t \leq T - 1$. We have 
\begin{align*}
    \lVert W_T \rVert^N \leq & A \Big( \sum_{t = 0}^{T - 1} (1 + \alpha t^{\frac{N}{N + 2}} + \beta)^{\frac{N-2}{N}} \Big)^{1/2} + \beta \\
    \overset{\frac{N-2}{N} < 1}{\leq} & A \Big( \sum_{t = 0}^{T - 1} (1 + \beta)^{\frac{N-2}{N}} + (\alpha t^{\frac{N}{N + 2}} )^{\frac{N-2}{N}} \Big)^{1/2} + \beta \\
    = & A \Big( \sum_{t = 0}^{T - 1} ((1 + \beta)^{\frac{N-2}{N}} + \alpha^{\frac{N-2}{N}}  t^{\frac{N - 2}{N + 2}}) \Big)^{1/2} + \beta \\
    \leq & A \Big( T(1 + \beta)^{\frac{N-2}{N}} + T \cdot \alpha^{\frac{N-2}{N}}  T^{\frac{N - 2}{N + 2}} \Big)^{1/2} + \beta \\
    = & A \Big( T(1 + \beta)^{\frac{N-2}{N}} + \alpha^{\frac{N-2}{N}}  T^{\frac{2N}{N + 2}} \Big)^{1/2} + \beta \\
    \overset{\frac{2N}{N + 2} > 1}{\leq} & A \Big( ((1 + \beta)^{\frac{N-2}{N}} + \alpha^{\frac{N-2}{N}})  T^{\frac{2N}{N + 2}} \Big)^{1/2} + \beta \\
    = & A ((1 + \beta)^{\frac{N-2}{N}} + \alpha^{\frac{N-2}{N}})^{1/2}  T^{\frac{N}{N + 2}}  + \beta \\
    \leq & A ((1 + \beta)^{\frac{N-2}{2N}} + \alpha^{\frac{N-2}{2N}}) T^{\frac{N}{N + 2}}  + \beta.
\end{align*}
To achieve the proof, we need to show there exists $\alpha$ such that 
\begin{align*}
A ((1 + \beta)^{\frac{N-2}{2N}} + \alpha^{\frac{N-2}{2N}}) T^{\frac{N}{N + 2}}  + \beta \leq & \alpha T^{\frac{N}{N + 2}} + \beta \\
A ((1 + \beta)^{\frac{N-2}{2N}} + \alpha^{\frac{N-2}{2N}}) \leq \alpha.
\end{align*}
Since the right hand side grows linearly in $\alpha$ while the left hand side is sublinear, growing as $\alpha^{\frac{N-2}{2N}}$, we are guaranteed to find $\alpha^*$ large enough to satisfy this condition. We therefore have that $\sum_{t = 0}^{T - 1} \frac{1}{Q_t} $ is lower bounded by a $p$-series that grows as $T^{1/L}$, which also lower bounds $\sum_{t=0}^{T-1} \eta_t$.
\begin{align*}
    \sum_{t=0}^{T-1} \eta_t \geq \sum_{t = 0}^{T - 1} \frac{1}{Q_t} \geq \sum_{t = 0}^{T - 1} \frac{1}{\rho (1 + \mathcal{L}(W_0)^{1/2})  
    (2L - 1)(1 + \alpha^* t^{\frac{L - 1}{L}
    } + \beta )} = \Theta(T^{1/L}).
\end{align*}

We therefore have that 
\begin{align*}
    \min_{t=0,...,T - 1} \lVert \nabla \mathcal{L} (W_t) \rVert^2 \leq & \frac{2L \mathcal{L}(W_0)}{ \sum_{t = 0}^{T - 1} \eta_t} = O(\frac{1}{T^{1/L}}),
\end{align*}
which concludes the proof.

\subsection{Bounds on $z_{\ell}(W)$, $h_{\ell}(W)$, and $\delta_{\ell}(W)$}

\begin{lemma}
\label{lemma:zbound}
    For $z_{\ell}(W)$ defined in (\ref{eq:definez}), we have for all $W \in \mathbb{R}^{D_2 \times D_1}$,
    \begin{equation*} 
        \lVert z_\ell(W) \rVert \leq c_x d_{max}^{1/2} c_{max}^{\ell-1} \sum_{i = 1}^{\ell} \lVert W \rVert^i.
    \end{equation*}
\end{lemma}
\begin{proof}
For the base case $\ell = 1$, we have
\begin{align*}
    \lVert z_1(W) \rVert = \lVert E_1[W] x_i \rVert \leq \lVert W \rVert \lVert x_i \rVert \leq \lVert W \rVert c_x d^{1/2}.
\end{align*}
For the general case, we have
\begin{align*}
    \lVert z_\ell(W) \rVert = & \lVert E_{\ell}[W] \sigma_{\ell - 1}(z_{\ell - 1}(W)) \rVert, \\
    \leq & \lVert W \rVert \lVert \sigma_{\ell - 1}(z_{\ell - 1}(W)) \rVert, \\
    \overset{\text{Lemma \ref{lemma:elementwisesigma}}}{\leq} & \lVert W \rVert (c_0 d^{1/2}_{\ell - 1} + c_1 \lVert z_{\ell - 1}(W) \rVert), \\
    \leq & c_x d_{max}^{1/2} c_{max}^{\ell-1} \sum_{i = 1}^{\ell} \lVert W \rVert^i.
\end{align*}
\end{proof}

\begin{lemma} 
\label{lemma:zchange}
For all $W_1, W_2 \in \mathbb{R}^{D_2 \times D_1}$ and $\ell = 1,...,L$, we have
    \begin{align*}
        \lVert z_{\ell}(W_2) - z_{\ell}(W_1) \rVert
    \leq &  c_x d_{max}^{1/2} c_{max}^{\ell - 1} \lVert W_2 - W_1 \rVert \sum_{k = 0}^{\ell - 1} \sum_{j = 0}^{\ell - k - 1}  \lVert W_1 \rVert^j \lVert W_2 \rVert^{k}.
    \end{align*}
\end{lemma}

\begin{proof}

We first show the bound for $\ell = 1$.
\begin{align*}
    \lVert z_{1}(W_2) - z_{1}(W_1) \rVert = & \lVert E_1[W_2] x_i - E_1[W_1] x_i \rVert \\
    \leq & c_x d^{1/2} \lVert W_2 - W_1 \rVert \leq c_x d_{max}^{1/2} \lVert W_2 - W_1 \rVert.
\end{align*}

Then for general $\ell = 2,...,L$, we have
\begin{align*}
    \lVert z_{\ell}(W_2) - z_{\ell}(W_1) \rVert = & \lVert E_\ell[W_2] \sigma_{\ell-1}(z_{\ell-1}(W_2)) - E_\ell[W_1] \sigma_{\ell-1}(z_{\ell-1}(W_1)) \rVert, \\
    \leq & \lVert W_2 \rVert \lVert \sigma_{\ell-1}(z_{\ell-1}(W_2)) - \sigma_{\ell-1}(z_{\ell-1}(W_1)) \rVert + \lVert \sigma_{\ell-1}(z_{\ell-1}(W_1)) \rVert \lVert W_2 - W_1 \rVert, \\
    \leq & c_2 \lVert W_2 \rVert \lVert z_{\ell-1}(W_2) - z_{\ell-1}(W_1) \rVert +  \lVert \sigma_{\ell-1}(z_{\ell-1}(W_1)) \rVert \lVert W_2 - W_1 \rVert, \\
    \overset{\text{Lemma \ref{lemma:elementwisesigma}}}{\leq} & c_2 \lVert W_2 \rVert \lVert z_{\ell-1}(W_2) - z_{\ell-1}(W_1) \rVert +  (c_0 d_{\ell - 1}^{1/2} + c_1\lVert z_{\ell - 1}(W_1) \rVert ) \lVert W_2 - W_1 \rVert, \\
    \overset{\text{Lemma \ref{lemma:zbound}}}{\leq} & c_2 \lVert W_2 \rVert \lVert z_{\ell-1}(W_2) - z_{\ell-1}(W_1) \rVert +  c_x d_{max}^{1/2} c_{max}^{\ell-1} \Big(\sum_{i = 0}^{\ell - 1} \lVert W_1 \rVert^i \Big) \lVert W_2 - W_1 \rVert, \\
    \leq & c_x d_{max}^{1/2} c_{max}^{\ell-1} \lVert W_2 - W_1 \rVert \sum_{k = 0}^{\ell - 1} \sum_{j = 0}^{\ell - k - 1}  \lVert W_1 \rVert^j \lVert W_2 \rVert^{k}. 
\end{align*}

\end{proof}

\begin{lemma}
    \label{lemma:hbound}
For $h_{\ell}(W)$ defined in (\ref{eq:definez}), we have for all $W \in \mathbb{R}^{D_2 \times D_1}$,
    \begin{equation*}
        \lVert h_\ell(W) \rVert \leq c_x d_{max}^{1/2} c_{max}^{\ell} \sum_{k = 0}^{\ell}   \lVert W \rVert^k.
    \end{equation*}
\end{lemma}
\begin{proof}
    For the base case $\ell = 0$, we have
    \begin{equation*}
        \lVert h_0(W) \rVert = \lVert x_i \rVert \leq d^{1/2} c_x.
    \end{equation*}
For the general case, we have by Lemma \ref{lemma:zbound},
\begin{align*}
    \lVert h_\ell(W) \rVert = & \lVert  \sigma_{\ell}(z_\ell(W)) \rVert, \\
    \overset{\text{Lemma \ref{lemma:elementwisesigma}}}{\leq} & c_0 d_\ell^{1/2} + c_1 \lVert z_\ell(W) \rVert, \\
    \leq & c_0 d_\ell^{1/2} + c_x d_{max}^{1/2} c_{max}^{\ell} \sum_{i = 1}^{\ell} \lVert W \rVert^i, \\
    \leq & c_x d_{max}^{1/2} c_{max}^{\ell} \sum_{k = 0}^{\ell}   \lVert W \rVert^k,
\end{align*}
using the fact that $c_x \geq 1$.
\end{proof}

\begin{lemma}
\label{lemma:deltabound}
For $\delta_\ell(W)$ defined in (\ref{eq:deltadefine}), we have for all $W \in \mathbb{R}^{D_2 \times D_1}$,
    \begin{equation*}
        \lVert \delta_\ell(W) \rVert \leq c_{max}^{L - \ell + 1} \lVert W \rVert^{L - \ell}.
    \end{equation*}
\end{lemma}

\begin{proof}

For the base case $\ell = L$, we have the following due to Assumption \ref{assump:lipschitz} and $d_L = 1$, 
\begin{align*}
    \lVert \delta_L (W) \rVert = & \lVert \sigma'_L(z_L (W)) \rVert \leq c_2.
\end{align*}
For general $\ell$, we have
    \begin{align*}
        \lVert \delta_\ell(W) \rVert  = &\lVert \sigma'_\ell(z_\ell(W)) \odot (E_{\ell + 1}[W]^T \delta_{\ell + 1}(W)) \rVert, \\
        \overset{\text{Lemma \ref{lemma:hademard}}}{\leq} &  c_2 \lVert E_{\ell + 1}[W]^T \delta_{\ell + 1}(W)  \rVert, \\
        \leq & c_2 \lVert W \rVert \lVert \delta_{\ell + 1}(W) \rVert,\\
        \leq &  c_{max}^{L - \ell + 1} \lVert W \rVert^{L - \ell}.
    \end{align*}
\end{proof}

\begin{lemma} 
\label{lemma:delta_change}
For all $W_1, W_2 \in \mathbb{R}^{D_2 \times D_1}$, we have
\begin{align*}
    \lVert \delta_L(W_2) - \delta_L(W_1) \rVert \leq &  \lVert z_L(W_2) - z_L(W_1) \rVert \leq c_x d_{max}^{1/2} c_{max}^L \lVert W_2 - W_1 \rVert \sum_{k = 0}^{L - 1} \sum_{j = 0}^{L - k - 1}  \lVert W_1 \rVert^j \lVert W_2 \rVert^{k}.
\end{align*}
and for $\ell = 1,.., L - 1$, we have
\begin{align*}
    \lVert \delta_\ell(W_2) - \delta_\ell(W_1) \rVert \leq & c_x d_{max}^{1/2} c_{max}^{2L} \lVert W_2 - W_1 \rVert \Big[\sum_{k = \ell}^{L}  \sum_{i = k - \ell}^{2k - 1 - \ell}  \lVert W_2 \rVert^i \sum_{j = L - k}^{L - i + k - \ell - 1} \lVert W_1 \rVert^j  + \sum_{k = 0}^{L-\ell-1} \lVert W_2 \rVert^{k}  \lVert W_1 \rVert^{L - k - \ell - 1} \Big]. 
\end{align*}
\end{lemma}
\begin{proof}
For the base case $\ell = L$, we have
\begin{align*}
\lVert \delta_L(W_2) - \delta_L(W_1) \rVert = & \lVert \sigma'_L(z_L(W_2)) - \sigma'_L(z_L(W_1)) \rVert \\
\overset{(\ref{eq:l2})}{\leq} & c_3 \lVert z_L(W_2) - z_L(W_1) \rVert, \\
\overset{\text{Lemma \ref{lemma:zchange}}}{\leq} & c_x d_{max}^{1/2} c_{max}^{L} \lVert W_2 - W_1 \rVert \sum_{k = 0}^{L - 1} \sum_{j = 0}^{L - k - 1}  \lVert W_1 \rVert^j \lVert W_2 \rVert^{k}.
\end{align*}
For $\ell \leq L - 1$, we have
\begin{align*}
    \lVert \delta_\ell(W_2) - \delta_\ell(W_1) \rVert = & \lVert \sigma'_\ell(z_\ell(W_2)) \odot (E_{\ell + 1}[W_2]^T \delta_{\ell + 1}(W_2)) - \sigma'_\ell(z_\ell(W_1)) \odot (E_{\ell + 1}[W_1]^T \delta_{\ell + 1}(W_1))\rVert \\
    \leq & \lVert \sigma'_\ell(z_{\ell}(W_2)) 
    \odot (E_{\ell + 1}[W_2]^T \delta_{\ell + 1}(W_2) - E_{\ell + 1}[W_1]^T \delta_{\ell + 1}(W_1))\rVert   \\
    &+ \lVert \sigma'_\ell(z_\ell(W_2)) - \sigma'_\ell(z_\ell(W_1)) \rVert \lVert E_{\ell + 1}[W_1]^T \delta_{\ell + 1}(W_1) \rVert \\
    \leq & c_2 \lVert E_{\ell + 1}[W_2]^T \delta_{\ell + 1}(W_2) - E_{\ell + 1}[W_1]^T \delta_{\ell + 1}(W_1)\rVert   \\
    &+ \lVert \sigma'_\ell(z_\ell(W_2)) - \sigma'_\ell(z_\ell(W_1)) \rVert \lVert E_{\ell + 1}[W_1]^T \delta_{\ell + 1}(W_1) \rVert \\
    \overset{(\ref{eq:l2})}{\leq} & c_2 \lVert E_{\ell + 1}[W_2]^T \delta_{\ell + 1}(W_2) - E_{\ell + 1}[W_1]^T \delta_{\ell + 1}(W_1)\rVert   \\
    &+ c_3 \lVert z_\ell(W_2) - z_\ell(W_1) \rVert \lVert W_1 \rVert \lVert \delta_{\ell + 1}(W_1) \rVert \\
    \leq &  c_2\lVert \delta_{\ell + 1}(W_1)\rVert \lVert W_2 - W_1 \rVert +  c_2 \lVert W_2 \rVert \lVert \delta_{\ell + 1}(W_2) - \delta_{\ell + 1}(W_1) \rVert   \\
    &+ c_3\lVert z_\ell(W_2) - z_\ell(W_1) \rVert \lVert W_1 \rVert \lVert \delta_{\ell + 1}(W_1) \rVert \\
    \overset{\text{Lemma \ref{lemma:deltabound}}}{\leq} & c_{max}^{L - \ell + 1}\lVert W_1 \rVert^{L - \ell - 1} \lVert W_2 - W_1 \rVert + c_2 \lVert W_2 \rVert \lVert \delta_{\ell + 1}(W_2) - \delta_{\ell + 1}(W_1) \rVert   \\
    &+ c_{max}^{L - \ell + 1} \lVert z_\ell(W_2) - z_\ell(W_1) \rVert \lVert W_1 \rVert \lVert W_1 \rVert^{L - \ell - 1} \\
    = & c_{max}^{L - \ell + 1} \lVert W_1 \rVert^{L - \ell - 1} \lVert W_2 - W_1 \rVert + c_2 \lVert W_2 \rVert \lVert \delta_{\ell + 1}(W_2) - \delta_{\ell + 1}(W_1) \rVert   \\
    &+ c_{max}^{L - \ell + 1} \lVert z_\ell(W_2) - z_\ell(W_1) \rVert  \lVert W_1 \rVert^{L - \ell } 
\end{align*}
Solving the recursive relationship leads to
\begin{align*}
    \lVert \delta_\ell(W_2) - \delta_\ell(W_1) \rVert \leq & c_{max}^{L} \sum_{k = \ell}^{L} \lVert W_2 \rVert^{k - \ell} \lVert z_k (W_2) - z_k (W_1) \rVert  \lVert W_1 \rVert^{L - k} + c_{max}^{L}\sum_{k = \ell}^{L-1} \lVert W_2 \rVert^{k - \ell} \lVert W_1 \rVert^{L - k - 1} \lVert W_2 - W_1 \rVert\\
    \overset{\text{Lemma \ref{lemma:zchange}}}{\leq} & c_x d_{max}^{1/2} c_{max}^{2L} \lVert W_2 - W_1 \rVert \sum_{k = \ell}^{L}  \lVert W_2 \rVert ^{k - \ell}( \sum_{i = 0}^{k - 1} \sum_{j = 0}^{k - i - 1} \lVert W_1 \rVert^j \lVert W_2 \rVert^i) \lVert W_1 \rVert^{L - k}  \\
    &+ c_{max}^{L} \sum_{k = \ell}^{L-1} \lVert W_2 \rVert^{k - \ell}  \lVert W_1 \rVert^{L - k - 1} \lVert W_2 - W_1 \rVert \\
    = & c_x d_{max}^{1/2} c_{max}^{2L} \lVert W_2 - W_1 \rVert \sum_{k = \ell}^{L}  \sum_{i = 0}^{k - 1} \sum_{j = 0}^{k - i - 1} \lVert W_1 \rVert^{L - k + j} \lVert W_2 \rVert^{k - \ell + i}  \\
    &+ c_{max}^{L} \sum_{k = \ell}^{L-1} \lVert W_2 \rVert^{k - \ell}  \lVert W_1 \rVert^{L - k - 1} \lVert W_2 - W_1 \rVert \\
    = & c_x d_{max}^{1/2} c_{max}^{2L} \lVert W_2 - W_1 \rVert \sum_{k = \ell}^{L}  \sum_{i = k - \ell}^{2k - 1 - \ell}  \lVert W_2 \rVert^i \sum_{j = L - k}^{L - i + k - \ell - 1} \lVert W_1 \rVert^j \\
    & + c_{max}^{L} \sum_{k = \ell}^{L-1} \lVert W_2 \rVert^{k - \ell}  \lVert W_1 \rVert^{L - k - 1} \lVert W_2 - W_1 \rVert \\
    = & c_x d_{max}^{1/2} c_{max}^{2L} \lVert W_2 - W_1 \rVert \sum_{k = \ell}^{L}  \sum_{i = k - \ell}^{2k - 1 - \ell}  \lVert W_2 \rVert^i \sum_{j = L - k}^{L - i + k - \ell - 1} \lVert W_1 \rVert^j \\
    & +c_{max}^{L} \sum_{k = 0}^{L-\ell-1} \lVert W_2 \rVert^{k}  \lVert W_1 \rVert^{L - k - \ell - 1} \lVert W_2 - W_1 \rVert \\
    \leq & c_x d_{max}^{1/2} c_{max}^{2L} \lVert W_2 - W_1 \rVert \Big[\sum_{k = \ell}^{L}  \sum_{i = k - \ell}^{2k - 1 - \ell}  \lVert W_2 \rVert^i \sum_{j = L - k}^{L - i + k - \ell - 1} \lVert W_1 \rVert^j  + \sum_{k = 0}^{L-\ell-1} \lVert W_2 \rVert^{k}  \lVert W_1 \rVert^{L - k - \ell - 1} \Big].
\end{align*}
    
\end{proof}

\subsection{Lipschitz properties of common activation and loss functions}
\label{append:Lipproperties}

In the following, we demonstrate that the following popular activation and loss functions satisfy the assumptions of this work.

\paragraph{Linear activation function.} The linear activation function is defined as 
\begin{equation*}
    \sigma(x) = x.
\end{equation*}
This is linearly bounded with $c_0 = 0$ and $c_1 = 1$. Since $\sigma'(x) = 1$ and $\sigma''(x) = 0$, it is Lipschitz continuous and Lipschitz smooth.

\paragraph{Tanh activation function.}  The tanh activation function is defined as 
$$\sigma(x) = \frac{e^x - e^{-x}}{e^x + e^{-x}}$$
It is uniformly bounded $|\sigma(x)| \leq 1$ so it is also linearly bounded. We also have that $|\sigma'(x)| \leq 1$ and $|\sigma''(x) | \leq \frac{4}{3 \sqrt{3}} $ so it is Lipschitz continuous and Lipschitz smooth.

\paragraph{Sigmoid activation function.} The sigmoid activation function is defined as 
$$\sigma(x) = \frac{1}{1 + e^{-x}}$$
It is is uniformly bounded $|\sigma(x)|\leq 1$ so it is also linearly bounded. We also have that $|\sigma'(x)| \leq \frac{1}{4}$ and $|\sigma''(x)| \leq 
\frac{1}{6 \sqrt{3}}$, so it is Lipschitz continuous and Lipschitz smooth.

\paragraph{Softplus activation function.} The softplus activation function is defined as 
$$\sigma(x) = \ln(1 + e^x).$$
We have that $|\sigma(x)| \leq \ln(2) + |x|$, so it is linearly bounded. We also have $|\sigma'(x)| \leq 1$ and $|\sigma''(x)| \leq \frac{1}{4}$, so it is Lipschitz continuous and Lipschitz smooth.

\paragraph{Gaussian Error Linear Unit (GELU) activation function.} The GELU function \citep{gelu} weighs an input by its value and cumulative probability under a standard normal distribution, resulting in a smoothed alternative to ReLU activation functions. The exact definition using the standard normal cumulative distribution function $\Phi(x)$ is as follows,
$$\sigma(x) = x\Phi(x) = \frac{1}{2}x\left(1 + \text{erf}\left(\frac{x}{\sqrt{2}}\right)\right).$$
We have that $|\sigma(x)| \leq |x|$, so it is linearly bounded. We also have
\begin{align*}
    |\sigma'(x)| = | \frac{1}{2} (\text{erf}(\frac{x}{\sqrt{2}}) + 1) + \frac{x e^{-x^2/2}}{\sqrt{2 \pi}}|\leq \frac{1}{2}(\text{erf}(1) + 1 + \frac{2}{e \sqrt{\pi}}) \approx 1.1289,
\end{align*}
and 
\begin{align*}
    |\sigma''(x)| = |- \frac{e^{-x^2/2}(x^2 - 2)}{\sqrt{2 \pi}} |  \leq \sqrt{2/\pi},
\end{align*}
so it is Lipschitz continuous and Lipschitz smooth.

\paragraph{Sigmoid Linear Unit (SiLU) activation function.} The SiLU function \citep{ELFWING20183}, also known as Swish, is another smooth alternative to ReLU which is defined as the product of $x$ and the sigmoid function as follows,
\begin{equation*}
    \sigma(x) = x\big(\frac{1}{1 + e^{-x}}\big).
\end{equation*}
We have that $|\sigma(x)| \leq x$, so it is linearly bounded. We also have that 
$|\sigma'(x)|$ achieves a maximum value of approximately 1.0998 and $|\sigma''(x)| \leq \frac{1}{4}$. So this function is Lipschitz continuous and Lipschitz smooth. 
\paragraph{Huber loss.}
The Huber loss function is a differentiable function that combines the properties of mean squared error and mean absolute error.
For parameter $\tau > 0$, the Huber loss function is defined as
\begin{align*}
    J_i(u) = \begin{cases} 
\frac{1}{2}(u - y_i)^2 & \text{if } |u - y_i| \le \tau \\ 
\tau |u - y_i| - \frac{1}{2}\tau^2 & \text{otherwise}. 
\end{cases}
\end{align*}
The derivative of $J_i$ is the following continuous function
\begin{align*}
    J_i'(u) = \begin{cases} 
u - y_i & \text{if } |u - y_i| \le \tau \\ 
\tau sgn(u - y_i) & \text{otherwise}. 
\end{cases}
\end{align*}
We note that the second derivative is uniformly bounded everywhere such that $|J''_i(u)| \leq 1$ for all $u \in \mathbb{R}$. So Huber loss satisfies Assumption \ref{assump:lipJ} with $c_J = 1$.

\paragraph{Binary cross-entropy loss.}
We have the following extended derivation from the main body.
\begin{align*}
    J_i(u) = & - \big[ y_i \log (\frac{1}{1 + e^{-u}}) + (1 - y_i) \log(1 - \frac{1}{1 + e^{-u}}) \big],\\
    = & y_i \log (1 + e^{-u}) - (1 - y_i) \log( \frac{e^{-u}}{1 + e^{-u}})\\
    = & y_i \log (1 + e^{-u}) + (y_i - 1) [\log(e^{-u}) -\log(1 + e^{-u})]\\
    = & (y_i - 1)(-u) + \log(1 + e^{-u})\\
    = & - u y_i + \log (e^u) + \log(1 + e^{-u})\\ 
    = & - u y_i + \log(1 + e^u).
\end{align*}

\section{Experimental Details and Additional Results}
\label{sec:experimentappend}

Code for all experiments is open-sources at the Github repository at \url{https://github.com/siqiaomu/nn_optimization}. Table \ref{tab:hyperparameters} displays the hyperparameters used in our experimental results. Code is open-sourced at the Github repository.

\begin{table}[h]
\caption{Experiment hyperparameters.}
\label{tab:hyperparameters}
\begin{center}
\begin{tabular}{ll}
\multicolumn{1}{c}{\bf Hyperparameter}  &\multicolumn{1}{c}{\bf Value}
\\ \hline \\
Batch size         &512 \\
Activation function & GeLU \\
Input data dimension & 3072 (32 $\times$ 32 $\times$ 3) \\
Number of epochs & 100\\
$\alpha$             &0.75 \\
Constant learning rates & [0.05, 0.07, 0.1]\\
\end{tabular}
\end{center}
\end{table}

\paragraph{Learning rate.} To determine a fair range of constant learning rates that balances speed and stability, we use a cyclical learning rate finder as introduced in \cite{smith_lr} and further described in \cite{ahmed_lr}. We start with a small learning rate and train the model for a single epoch, repeating this with exponentially increasing learning rates. We plot the loss over the learning rates and pick learning rates in the range where the loss decreases the fastest. Figure \ref{fig:learning_rate} displays the result.

\begin{figure}[h]
    \centering
\includegraphics[width=0.7\linewidth]{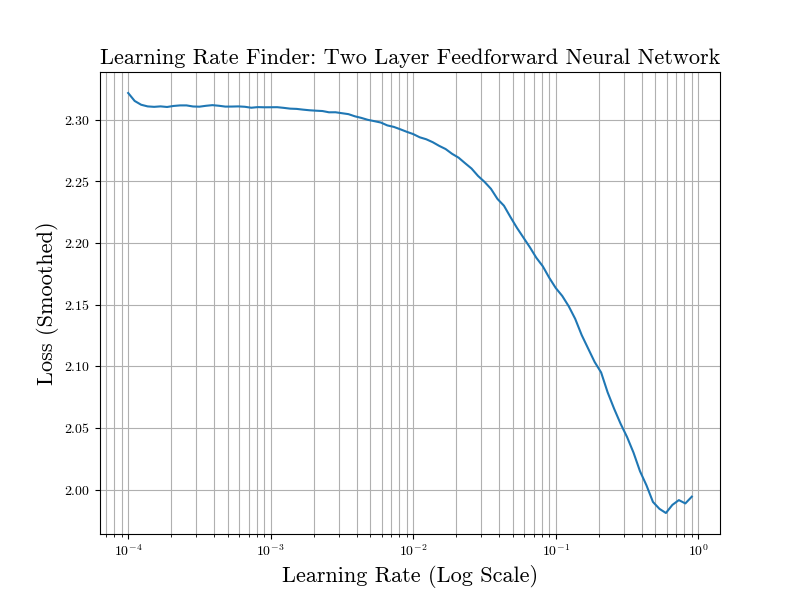}
    \caption{Learning rate selection using a cyclical finder for a two-layer neural network with GELU activations.}
    \label{fig:learning_rate}
\end{figure}

\paragraph{Hardware and software.} All experiments were run using PyTorch 2.5.0 and CUDA 12.1, on an Intel(R) Xeon(R) Silver 4208 CPU (2.10GHz) with an  NVIDIA RTX A6000 GPU.

\subsection{Full comparison between constant and adaptive learning rates}

Figure \ref{fig:bigcomparison} displays the full results of the shared-weight model experiments for a range of $L$. As discussed in the main body, the adaptive learning rate displays faster and more stable convergence than constant learning rates in the same range of values. Moreover, as the number of layers increase, the iterates undergo a longer initial plateau period, corresponding to a flatter region around initialization. This further supports the need for a larger learning rate early on to accelerate the iterates out of this region.

\begin{figure}[h]
    \centering
    \includegraphics[width=\linewidth]{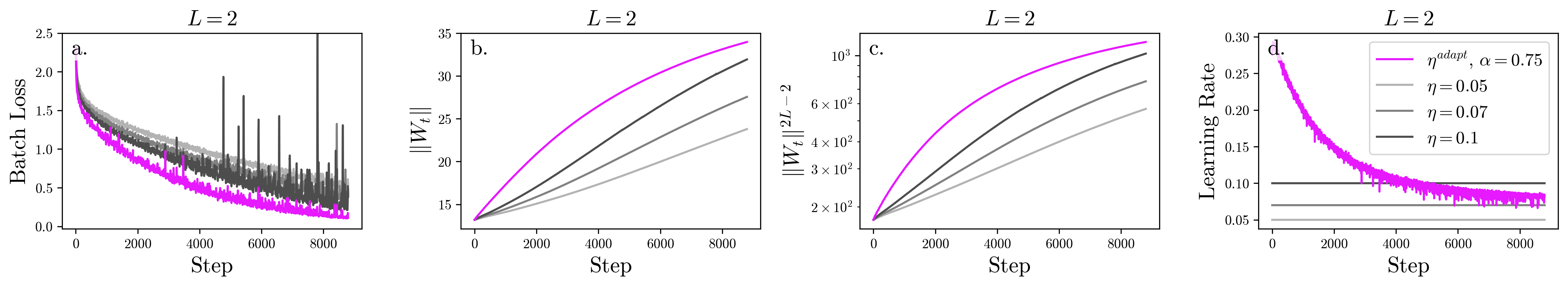}
    \includegraphics[width=\linewidth]{standard_gelu_L4_schedule_comparison.png}

    \includegraphics[width=\linewidth]{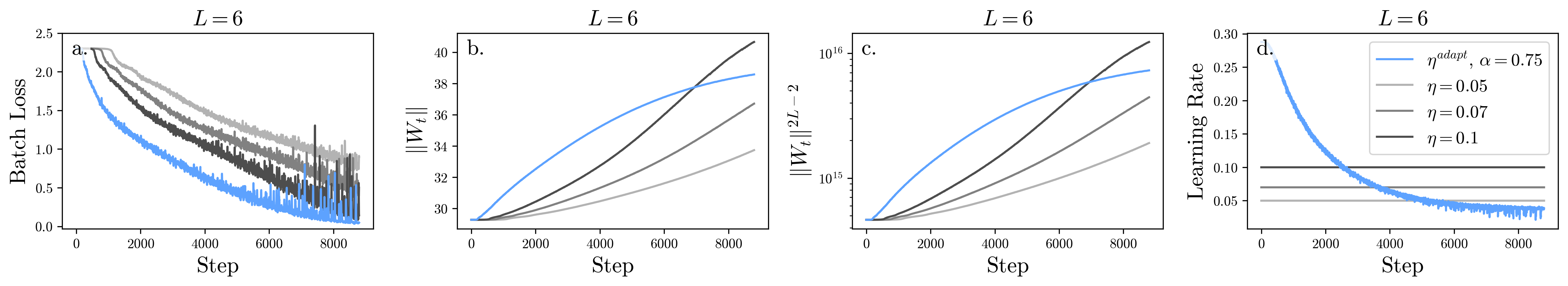}

    \includegraphics[width=\linewidth]{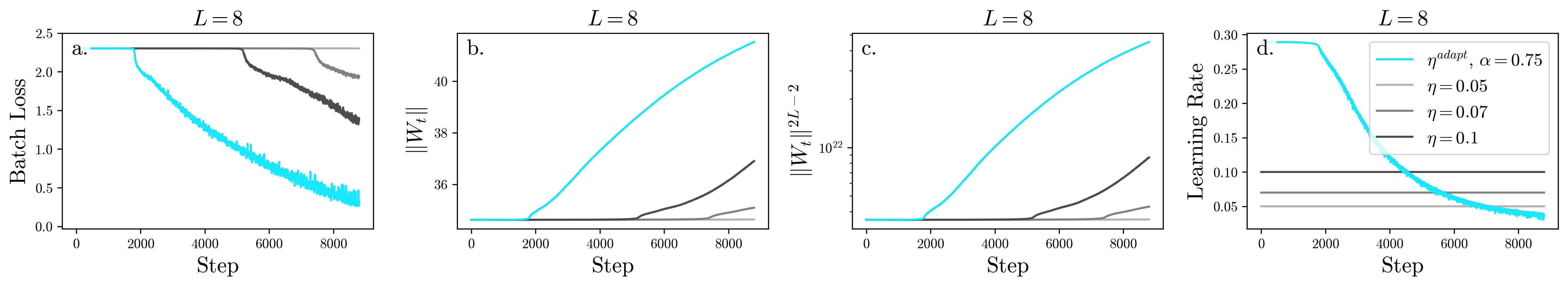}
    \caption{Comparison of constant and adaptive learning rates for standard neural networks of varying depths.}
    \label{fig:bigcomparison}
\end{figure}

\end{document}